\documentclass{article}

\usepackage{iclr2027_conference,times}
\usepackage[utf8]{inputenc}
\usepackage[T1]{fontenc}
\usepackage{booktabs}
\usepackage{graphicx}
\usepackage{amsmath}
\usepackage{amssymb}
\usepackage{amsthm}
\usepackage{algorithm}
\usepackage{float}
\usepackage{microtype}
\usepackage[table]{xcolor}
\usepackage{tcolorbox}
\usepackage{pgfplots}
\pgfplotsset{compat=1.18}
\usepackage{placeins}
\usepackage{url}
\usepackage{hyperref}

\newtheorem{proposition}{Proposition}[section]
\definecolor{propositionpurple}{HTML}{8A62A6}
\definecolor{algorithmcomment}{HTML}{347C7D}
\definecolor{algorithmfunction}{HTML}{D63384}
\definecolor{rankblue}{HTML}{3182BD}
\definecolor{samplegreen}{HTML}{4A8F2A}
\definecolor{pilotsink}{HTML}{2A78D6}
\definecolor{pilotkl}{HTML}{EB6834}
\definecolor{pilotmmd}{HTML}{1BAF7A}
\newtcolorbox{purplepropositionbox}{
  colback=propositionpurple!7!white,
  colframe=propositionpurple,
  boxrule=0.9pt,
  arc=2pt,
  boxsep=2pt,
  left=2pt,
  right=2pt,
  top=2pt,
  bottom=2pt,
  before skip=8pt,
  after skip=8pt
}

\title{One-Step Generation via \\ Riemannian Wasserstein Gradient Flows}

\author{%
David Li$^{1}$\thanks{Work done during an internship at KAIST. Correspondence to: David Li \texttt{<david.li@mbzuai.ac.ae>}.} \quad Chanhyuk Lee$^{2}$ \quad Jaehoon Yoo$^{2}$ \quad
Nikita Gushchin$^{3,5}$ \quad Floor Eijkelboom$^{4}$ \\
\bfseries Eric Moulines$^{1,6}$ \quad Maxim Panov$^{1}$ \quad Alexander Korotin$^{3,5}$ \quad
Seunghoon Hong$^{2}$ \quad Jinwoo Kim$^{2}$ \\[4pt]
{\normalfont $^{1}$Mohamed bin Zayed University of Artificial Intelligence} \\
{\normalfont $^{2}$Korea Advanced Institute of Science and Technology (KAIST)} \\
{\normalfont $^{3}$Applied AI Institute} \\
{\normalfont $^{4}$University of Amsterdam} \\
{\normalfont $^{5}$AXXX} \\
{\normalfont $^{6}$EPITA}%
}

\iclrfinalcopy

\begin{document}

\maketitle

\begin{abstract}
Recently, Drifting Models and Wasserstein Gradient Flows have attracted substantial attention because they move iterative distributional refinement to training and amortize it into a generator, enabling fast inference. However, existing formulations have been developed largely for continuous Euclidean domains, such as image spaces, where particles admit unconstrained additive updates. On constrained spaces, these updates can leave the valid domain or ignore its geometry, making them unsuitable targets for training. Recent work has adapted updates to these spaces, but has focused on particular fields or offered limited empirical comparison. We derive and compare several geometry-aware fields within a common training framework for one-step generators. We test the method on data with different structures and obtain competitive one-step results in each setting. The best-performing field varies by task, showing why the choice of objective matters in practice.
\end{abstract}

\begin{figure}[ht]
  \centering
  \includegraphics[width=\linewidth]{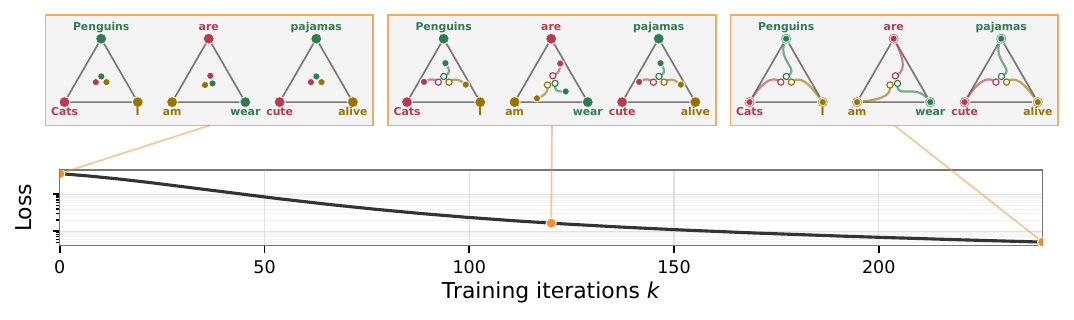}
  \caption{\textbf{Fixed-$z$ generation on a product of
  simplices.} Outputs lie in
  $\mathcal M =(\Delta^2)^3$, with one simplex per word position. Each color
  tracks $f_{\theta_k}(z_i)$ for a fixed $z_i$ as $\theta_k$ is updated,
  ending at
  \textcolor[HTML]{B83A4B}{\emph{Cats are cute}},
  \textcolor[HTML]{2E7D4F}{\emph{Penguins wear pajamas}}, and
  \textcolor[HTML]{9A7200}{\emph{I am alive}}.}
  \label{fig:simplex-wgf-overview}
\end{figure}

\section{Introduction}
\label{sec:introduction}

Diffusion and flow-based generators achieve strong sample quality by gradually
transporting a simple reference distribution toward the data
distribution~\citep{ho2020denoising,song2021scorebased,lipman2023flow}.  This
transport is often realized through many sequential updates, making inference
expensive.  Drifting Models offer a complementary strategy: they update the
model distribution iteratively during training and amortize the resulting
moves into a generator that can be evaluated directly at inference
time~\citep{deng2026drifting}.  Recent works connect such updates to
Wasserstein gradient flows (WGFs), making it possible to derive the drift from
an explicit objective
~\citep{gretton2026interpretation,han2026one}.  Most such formulations,
however, rely on Euclidean velocities and additive updates.
Many data types and representations live naturally on curved or constrained
spaces, where Euclidean updates need not follow the intrinsic geometry or even
remain in the state space.

Classical WGF theory extends beyond Euclidean spaces
~\citep{ambrosio2008gradient,erbar2010heat}, and recent generative and drifting
methods construct dynamics on Riemannian manifolds
~\citep{chen2024riemannian,cao2026gradient,esteban2026kernel}.  However, they typically focus on a single objective or evaluate alternatives in limited settings.  We address this gap by deriving a common formulation of objective-driven updates within a common Riemannian framework and comparing the resulting one-step generators across diverse settings.  We call the resulting framework \emph{Riemannian Wasserstein Gradient Flows} (R-WGF).

We start from an explicit objective measuring the discrepancy between the
model and target distributions.  Its first variation describes how this
objective changes, and the Riemannian metric converts this information into an
intrinsic particle velocity.  This velocity defines a continuous flow of the
model distribution.  Under explicit regularity assumptions, we show that the
flow decreases the chosen objective.  We then approximate the flow with
geometry-compatible finite updates and, for the objectives considered,
estimate these updates from samples and amortize them into a neural generator.

Our contributions are:
\begin{itemize}
  \item We formulate objective-driven drifting through Riemannian Wasserstein
  calculus, making explicit how a law-level functional and a ground metric
  determine the particle velocity.

  \item Under explicit regularity assumptions, we establish the population
  energy-dissipation identity.  We also provide geometry-compatible finite
  updates through Riemannian retractions and prove their local first-order
  consistency with the exponential map.

  \item We apply our Riemannian Wasserstein gradient-flow framework to three distinct domains, demonstrating one-step generation on spherical geospatial data, Binary-MNIST and discretized CIFAR-10 images, and LM1B text.
\end{itemize}


\section{Background}
\label{sec:background}

\subsection{Riemannian geometry}
\label{sec:background:riemannian}

Let $(\mathcal M,\mathsf g)$ be a smooth connected Riemannian manifold.  At each point $x\in\mathcal M$,
the tangent space $T_x\mathcal M$ contains the directions in which a sample
can move.  The metric $\mathsf g_x$ defines the inner product over $T_x\mathcal M$ denoted as
$\mathsf g_x(u,v)$ and norm $\|u\|_{\mathsf g,x}$ of tangent vectors
$u,v\in T_x\mathcal M$.  We write $\mathrm{vol}_{\mathsf g}$ for the Riemannian volume measure
and $\operatorname{div}_{\mathsf g}$ for the corresponding divergence.

On a general manifold, a tangent vector cannot in general be added to a
point.  Instead, we follow a geodesic, a constant-speed curve whose short
segments are shortest paths under the metric $\mathsf g$. The length of such a segment gives the geodesic distance $d_{\mathcal M}$ between its endpoints. For a starting
point $x\in\mathcal M$ and an initial velocity $v\in T_x\mathcal M$, let
$\gamma_{x,v}(t)$ denote this geodesic, with $\gamma_{x,v}(0)=x$ and
$\dot\gamma_{x,v}(0)=v$.  The exponential map turns the tangent vector into
a point on the manifold:
$\operatorname{Exp}_x(v):=\gamma_{x,v}(1)$ for sufficiently small $v$.
Thus $\operatorname{Exp}_x(sv)=\gamma_{x,v}(s)$ for sufficiently small $s$.

In the Wasserstein gradient flow below, a differentiable scalar field
$\phi:\mathcal M\to\mathbb R$ will act as the potential induced by the
distributional objective. The differential $d\phi_x:T_x\mathcal M\to\mathbb R$ is a linear functional on $T_x\mathcal M$ and describes the local change of
$\phi$ in every tangent direction.
The metric $\mathsf g_x$ identifies it with a unique tangent
vector.  This vector is the Riemannian gradient
$\operatorname{grad}_{\mathsf g}\phi(x)\in T_x\mathcal M$ satisfying
\begin{equation}
  \mathsf g_x\bigl(\operatorname{grad}_{\mathsf g}\phi(x),u\bigr)
  =d\phi_x[u]
  \qquad\text{for every }u\in T_x\mathcal M.
  \label{eq:riemannian-gradient-definition}
\end{equation}

Euclidean space is a special case: for $\mathcal M=\mathbb R^D$ with the
standard metric, $T_x\mathcal M=\mathbb R^D$,
$\mathsf g_x(u,v)=u^\top v$, and $\|u\|_{\mathsf g,x}$ is the usual
Euclidean norm.  In this case, the Riemannian gradient is simply the ordinary
gradient,
$\operatorname{grad}_{\mathsf g}\phi(x)=\nabla\phi(x)$.
Since $T_x\mathcal M=\mathbb R^D$, tangent vectors can also be added directly
to points, so that $\operatorname{Exp}_x(sv) = x+sv$.

\subsection{Drifting Models}
\label{sec:background:drifting}

Drifting Models were originally formulated for samples in Euclidean
space~\citep{deng2026drifting}.  Let $p$ be the data density and $p_0$ the
latent density.  Given $z\sim p_0$, the generator produces
$x=f_\theta(z)$.  We denote the density of generated samples by
$q_\theta=(f_\theta)_\#p_0$.  A sample-based drift field
$\widehat V_{p,q_\theta}$, estimated from real and generated samples, moves
$x$ toward a refined target.  The target and the generator loss are
\begin{equation}
  \widetilde x
  =x+\eta\widehat V_{p,q_\theta}(x),
  \qquad
  \mathcal L(\theta)
  :=\mathbb E_{z\sim p_0}\!\left[
    \left\|f_\theta(z)-\operatorname{sg}[\widetilde x]\right\|^2
  \right].
  \label{eq:background-drifting}
\end{equation}
Here $\eta>0$ is the step size and $\operatorname{sg}$ denotes
stop-gradient, so $\widetilde x$ is fixed while optimizing
$\mathcal L$.  Repeating these training stages transfers the sequence of
sample refinements into the generator. Therefore, at inference, sampling requires only
one evaluation of $f_\theta$.

Both parts of Eq.~\eqref{eq:background-drifting} use Euclidean geometry: $\widetilde x$ is obtained by vector addition, and the generator is fitted with a
squared Euclidean distance.  The construction explains how to amortize a
given drift, but it does not determine which population velocity
$\widehat V_{p,q_\theta}$ should approximate.  Wasserstein gradient flows
answer this question by deriving the velocity from an explicit functional of
the model distribution.

\subsection{Wasserstein gradient flows}
\label{sec:background:wgf}

Let $p,q$ be continuous target and model
probability densities, so that each integrates to one and has finite second
moment.
A Wasserstein gradient flow evolves $q$ by steepest descent of a
distributional objective $\mathcal F_p$ in the quadratic Wasserstein
geometry~\citep{jordan1998variational,otto2001geometry,
ambrosio2008gradient}.  We assume that the objective identifies the target:
\begin{equation}
  \operatorname*{arg\,min}_q\mathcal F_p(q)=\{p\}.
  \label{eq:background-functional}
\end{equation}
For an admissible perturbation $h$ with
$\int_{\mathbb R^D}h(x)\,dx=0$ and $q+sh$ a probability density for all
sufficiently small $s\geq0$, the first variation
$\psi_{p,q}:=\delta\mathcal F_p(q)/\delta q$ is defined by
\begin{equation}
  \left.\frac{\mathrm d}{\mathrm ds}
  \mathcal F_p(q+s h)\right|_{s=0^+}
  =\int_{\mathbb R^D}\psi_{p,q}(x)h(x)\,dx.
  \label{eq:background-first-variation}
\end{equation}
Thus $\psi_{p,q}(x)$ measures how the objective changes when the model
density is perturbed near $x$.

If the first variation is sufficiently regular, the WGF in Euclidean space
can be characterized as
\begin{equation}
  V_{p,q}(x)=-\nabla\psi_{p,q}(x),
  \qquad
  \partial_tq_t+\nabla\!\cdot\!\left(q_tV_{p,q_t}\right)=0.
  \label{eq:background-wgf}
\end{equation}
Here $\mathcal F_p$ specifies what should decrease, while $V_{p,q}$ specifies
how samples move.  For a sufficiently regular solution, the exact flow
satisfies the dissipation identity
\begin{equation}
  \frac{\mathrm d}{\mathrm dt}\mathcal F_p(q_t)
  =-\int_{\mathbb R^D}
  \left\|\nabla\psi_{p,q_t}(x)\right\|^2q_t(x)\,dx
  \leq 0.
  \label{eq:background-euclidean-dissipation}
\end{equation}
Thus the population objective cannot increase along the WGF.  A forward
Euler step is
\begin{equation}
  \widetilde x
  =x+\eta V_{p,q}(x)
  =x-\eta\nabla\psi_{p,q}(x).
  \label{eq:background-wgf-step}
\end{equation}
Taking $q=q_\theta$ and replacing $V_{p,q_\theta}$ by its sample estimate
$\widehat V_{p,q_\theta}$ recovers the teacher update in
Eq.~\eqref{eq:background-drifting}; the MSE loss then amortizes this step into
the generator.  Thus an objective-driven Drifting Model can be viewed as an
amortized Euclidean WGF step.  An arbitrary heuristic drift, however, need
not correspond to a law-level functional and so does not automatically
inherit this dissipation guarantee.

Wasserstein gradient flows on Riemannian manifolds are developed in
Section~\ref{sec:method}.  Recent WGF interpretations of drifting and
objective-specific fields are reviewed in Section~\ref{sec:related-work}.

\section{Riemannian Wasserstein Gradient Flows}
\label{sec:method}

We work on the Riemannian sample space $(\mathcal M,\mathsf g)$ introduced in
Section~\ref{sec:background:riemannian}, with geodesic distance
$d_{\mathcal M}$ and volume element $d\mathrm{vol}_{\mathsf g}$.  Let
$\mathcal P_2(\mathcal M)$ denote the space of probability distributions on
$\mathcal M$ with finite second moments, equipped with the quadratic
Wasserstein distance $W_{2,\mathsf g}$ induced by $d_{\mathcal M}$.  Throughout
this section, $p,q$ denote the target and model
distributions, respectively.  For fixed $p$, consider a law-level objective
whose minimum is attained at $q=p$:
\begin{equation*}
  \mathcal F_p:\mathcal P_2(\mathcal M)\to\mathbb R,
  \qquad
  \{p\}=\operatorname*{arg\,min}_{q}\mathcal F_p(q).
\end{equation*}

Our goal is to learn a generative model whose output distribution
$q$ approximates the target $p$ while retaining fast sampling.  Motivated by
promising recent results in WGF-based generative modeling
~\citep{gretton2026interpretation,han2026one}, we follow this paradigm by
decreasing $\mathcal F_p(q)$.  Because Euclidean updates need not respect
curved or constrained sample spaces, we formulate the flow intrinsically on a
Riemannian manifold, whose metric converts the first variation of
$\mathcal F_p$ into sample refinements.  We iterate these refinements during
training and amortize them into the generator, leaving inference to a single
evaluation.

The section proceeds in four steps.
First, we convert the first variation of $\mathcal F_p$ into an
intrinsic drift of individual samples.  Second, we establish the associated
population energy-dissipation identity.  Third, we instantiate
$\mathcal F_p$ with concrete objectives and derive the corresponding
intrinsic drifts.
Fourth, we replace the ideal flow by geometry-compatible finite steps and
amortize those steps into a generator.
The construction is independent of a particular ground geometry. Detailed
Euclidean, spherical, and categorical realizations are collected in
Appendix~\ref{app:concrete-geometries}.

\subsection{From a law functional to a sample-space drift}
\label{sec:method:flow}

If it exists, a first variation of $\mathcal F_p$ at $q$ is a function
$\psi_{p,q}:\mathcal M\to\mathbb R$ satisfying, for every density perturbation
$h$ with $\int_{\mathcal M}h(x)\,d\mathrm{vol}_{\mathsf g}(x)=0$,
\begin{equation}
  \left.\frac{d}{ds}
  \mathcal F_p(q+sh)
  \right|_{s=0^+}
  =\int_{\mathcal M}\psi_{p,q}(x)h(x)\,
  d\mathrm{vol}_{\mathsf g}(x).
  \label{eq:first-variation-definition}
\end{equation}
For $\mathcal M=\mathbb R^D$ and
$d\mathrm{vol}_{\mathsf g}(x)=dx$, this is exactly
Eq.~\eqref{eq:background-first-variation}.

To relate the first variation to sample motion, let $q_s^u$ be the density
obtained by pushing $q$ forward under
$x\mapsto\operatorname{Exp}_x(su(x))$, where $u$ is a smooth, compactly
supported tangent field.  We assume the transport chain rule
\begin{equation}
  \left.\frac{d}{ds}\mathcal F_p(q_s^u)\right|_{s=0}
  =\int_{\mathcal M}d(\psi_{p,q})_x[u(x)]q(x)\,
    d\mathrm{vol}_{\mathsf g}(x).
  \label{eq:transport-chain-rule}
\end{equation}
Equation~\eqref{eq:first-variation-definition} alone does not imply this rule.

Equip tangent fields with the $q$-weighted inner product
\begin{equation}
  \langle u,w\rangle_{q,\mathsf g}
  :=\int_{\mathcal M}\mathsf g_x(u(x),w(x))q(x)\,
  d\mathrm{vol}_{\mathsf g}(x),
  \qquad
  \|u\|_{q,\mathsf g}^2:=\langle u,u\rangle_{q,\mathsf g}.
  \label{eq:law-tangent-norm}
\end{equation}
Our goal is to select, among tangent velocity fields $u$ with finite
$q$-weighted norm, the sample-space drift that decreases $\mathcal F_p$ most
rapidly at the current density $q$.  Equation~\eqref{eq:transport-chain-rule}
gives the instantaneous change produced by $u$, so it is natural to minimize
this rate over velocity fields.  This minimization alone is ill posed: the
rate is linear in $u$, and whenever it is negative, rescaling the same
direction makes it arbitrarily negative.  We therefore balance the decrease
against the kinetic cost $\tfrac12\|u\|_{q,\mathsf g}^2$, which fixes the scale
using the tangent geometry of quadratic Wasserstein space.  The resulting
regularized problem is the infinite-dimensional analogue of
$\operatorname*{arg\,min}_v
\{\langle\nabla f,v\rangle+\tfrac12\|v\|^2\}=-\nabla f$ and thus
selects the steepest-descent velocity.

\begin{purplepropositionbox}
\begin{proposition}[Riemannian steepest drift]
\label{prop:riemannian-steepest-drift}
If $\|\operatorname{grad}_{\mathsf g}\psi_{p,q}\|_{q,\mathsf g}<\infty$, then
the regularized steepest-drift problem has the unique solution
($q\text{-a.e.}$)
\begin{align}
  V_{p,q}
  &:=\underset{u:\,\|u\|_{q,\mathsf g}<\infty}{\arg\min}
  \left\{
    \int_{\mathcal M}d(\psi_{p,q})_x[u(x)]q(x)\,
    d\mathrm{vol}_{\mathsf g}(x)
    +\frac12\|u\|_{q,\mathsf g}^2
  \right\}
  =-\operatorname{grad}_{\mathsf g}\psi_{p,q}
  \label{eq:general-riemannian-velocity}
\end{align}
\end{proposition}
\end{purplepropositionbox}
The proof is given in Appendix~\ref{app:proof:steepest-velocity}.

Let $t\mapsto q_t$ be a sufficiently smooth density path transported by
$V_{p,q_t}$.  It satisfies the Riemannian continuity equation
\begin{equation}
  \partial_tq_t+\operatorname{div}_{\mathsf g}(q_tV_{p,q_t})=0.
  \label{eq:riemannian-continuity}
\end{equation}
For $\mathcal M=\mathbb R^D$ with the Euclidean metric,
$\operatorname{grad}_{\mathsf g}=\nabla$, recovering
$V_{p,q}=-\nabla\psi_{p,q}$ from Eq.~\eqref{eq:background-wgf}.

\paragraph{Result: the law-to-sample map.}
Proposition~\ref{prop:riemannian-steepest-drift} shows that the first variation
and the ground metric determine the intrinsic sample drift.
The equality holds for $q$-almost every $x$.
Changing $\mathcal F_p$ changes the scalar field $\psi_{p,q}$, changing
$\mathsf g$ changes how its differential is converted into a tangent vector.
Together with Eq.~\eqref{eq:riemannian-continuity}, this field determines the
evolution of $q_t$.  Figure~\ref{fig:objective-geometry-interface} illustrates
this passage from objective descent to intrinsic particle motion.

\subsection{Population dissipation in the general geometry}
\label{sec:method:dissipation}

Having identified the velocity, we next verify that its population flow
decreases the functional from which it was derived.

\begin{purplepropositionbox}
\begin{proposition}[General population dissipation theorem]
\label{prop:population-dissipation}
Let $(q_t)_{t\in[0,T]}$ be a sufficiently smooth path of probability densities
with finite second moments that solves Eq.~\eqref{eq:riemannian-continuity} with
$V_{p,q_t}=-\operatorname{grad}_{\mathsf g}\psi_{p,q_t}$.  Assume that the
dynamic transport chain rule
\begin{equation*}
  \frac{d}{dt}\mathcal F_p(q_t)
  =\int_{\mathcal M}d(\psi_{p,q_t})_x[V_{p,q_t}(x)]q_t(x)\,
  d\mathrm{vol}_{\mathsf g}(x)
\end{equation*}
holds for almost every $t$.  Then, for almost every $t$,
\begin{equation}
  \frac{d}{dt}\mathcal F_p(q_t)
  =-\mathcal D_p(q_t)\leq0,
  \qquad
  \mathcal D_p(q)
  :=\int_{\mathcal M}
  \|\operatorname{grad}_{\mathsf g}\psi_{p,q}(x)\|_{\mathsf g,x}^2q(x)\,
  d\mathrm{vol}_{\mathsf g}(x),
  \label{eq:population-descent}
\end{equation}
\end{proposition}
\end{purplepropositionbox}
The proof and technical extensions are in
Appendix~\ref{app:proof:population-dissipation}.

\paragraph{Result: population descent.}
For every sufficiently regular solution satisfying the stated chain rule,
the objective is non-increasing at rate $\mathcal D_p(q_t)$, the squared
$q_t$-weighted norm of the first-variation gradient.

\subsection{Objective-specific drifts}
\label{sec:method:drifts}

The framework supports different law-level objectives, but obtaining an
explicit drift requires choosing $\mathcal F_p$ and computing its first
variation $\psi_{p,q}$.  Since
$V_{p,q}=-\operatorname{grad}_{\mathsf g}\psi_{p,q}$, the objective determines
the descent direction, while the ground geometry turns it into an intrinsic
tangent velocity.  Both choices therefore shape the sample dynamics and can
substantially affect generative performance.  In this work we consider three functionals: Sinkhorn divergence, MMD, and KL divergence. We also consider three ground geometries: Euclidean space $\mathbb R^D$, the round
sphere $\mathbb S^{D-1}$, and the product simplex
$(\Delta_\circ^{d-1})^m$ equipped with the Fisher--Rao metric.  To express
these drifts in a common notation,
$\nabla_x$ denotes an ambient-coordinate representative of the descent
covector $-d\psi_{p,q}(x)$.  Each drift is obtained in two steps: first, we
compute this ambient derivative and second, the ground metric maps it to a
velocity in $T_x\mathcal M$. In Euclidean space, this map is $I$, so the ambient direction is unchanged.
On the sphere, it is $P_x=I-xx^\top$, which removes the radial component. On the product simplex, it is $
  G_x=\operatorname{blkdiag}_{\ell=1}^m
  \bigl(\operatorname{Diag}(x_\ell)-x_\ell x_\ell^\top\bigr)$,
which performs the Fisher--Rao conversion blockwise and produces a zero-sum
tangent vector in every simplex factor.  The construction and geometric
interpretation of these maps are given in
Appendix~\ref{app:concrete-geometries}. Table~\ref{tab:objective-drifts} gives the resulting drift for each of the
nine objective--geometry pairs.

\begin{table}[ht]
  \caption{Intrinsic drift $V_{p,q}(x)$ for each objective--geometry pair
  considered in this work.  Definitions and drift derivations for Sinkhorn,
  MMD, and KL are given in Appendices~\ref{app:proof:sinkhorn-drift},
  \ref{app:proof:mmd-drift}, and~\ref{app:proof:kl-drift}.}
  \label{tab:objective-drifts}
  \centering
  \footnotesize
  \resizebox{\linewidth}{!}{%
  \begin{tabular}{@{}lccc@{}}
    \toprule
    \textbf{Manifold $\mathcal M$}
    & $\mathcal S_\varepsilon(q,p)$
    & $\tfrac12\operatorname{MMD}_\kappa^2(q,p)$
    & $\operatorname{KL}(q\,\|\,p)$ \\
    \midrule
    \textbf{Euclidean:} $\mathbb R^D$
    & $\nabla_x\!\left(f_{q,q}^\varepsilon-f_{q,p}^\varepsilon\right)$
    & $\nabla_x\!\left(m_p^\kappa-m_q^\kappa\right)$
    & $\nabla_x\log\!\frac{p(x)}{q(x)}$ \\
    \addlinespace[4pt]
    \textbf{Round sphere:} $\mathbb S^{D-1}$
    & $P_x\nabla_x\!\left(f_{q,q}^\varepsilon-f_{q,p}^\varepsilon\right)$
    & $P_x\nabla_x\!\left(m_p^\kappa-m_q^\kappa\right)$
    & $P_x\nabla_x\log\!\frac{p(x)}{q(x)}$ \\
    \addlinespace[4pt]
    \textbf{Product simplex:} $(\Delta_\circ^{d-1})^m$
    & $G_x\nabla_x\!\left(f_{q,q}^\varepsilon-f_{q,p}^\varepsilon\right)$
    & $G_x\nabla_x\!\left(m_p^\kappa-m_q^\kappa\right)$
    & $G_x\nabla_x\log\!\frac{p(x)}{q(x)}$ \\
    \bottomrule
  \end{tabular}%
  }
\end{table}
\FloatBarrier

\subsection{Practical discretization and neural amortization}
\label{sec:method:discretization}

The framework above describes how to move samples toward a target
distribution.  Our goal, however, is to train a model that can
generate from that distribution in a single forward pass.  We make this
practical by discretizing the flow, estimating the drifting field with finite number of samples and amortizing its updates into a neural
generator.

\paragraph{From the continuous drift to a finite update.}
The field $V_{p,q}$ specifies only the instantaneous motion of samples under
the current model law $q$, whereas training requires a finite target point
that remains on $\mathcal M$.  Fix the target law $p$, the current model law
$q$, a sample $x\in\mathcal M$, and a small step size $\eta>0$.  The vector
$V_{p,q}(x)\in T_x\mathcal M$ gives the sample velocity at $x$.  If this
velocity is held fixed during one step, the intrinsic reference update is
$\operatorname{Exp}_x(\eta V_{p,q}(x))$: for sufficiently small $\eta$, it is
the point reached by following the geodesic from $x$ in the direction
$V_{p,q}(x)$.  We call an update geometry-compatible when its discrepancy
from this reference is $O(\eta^2)$.

On a general manifold, $\operatorname{Exp}_x$ may have no closed-form
expression, and evaluating it can require numerically solving the geodesic
equation.  We therefore use a $C^2$ retraction, i.e., a local map
$\mathcal R_x:T_x\mathcal M\to\mathcal M$ defined near
$0\in T_x\mathcal M$, satisfying
\[
  \mathcal R_x(0)=x,
  \qquad
  \left.\frac{d}{ds}\mathcal R_x(sv)\right|_{s=0}=v
  \quad\text{for every }v\in T_x\mathcal M.
\]
Thus the retraction starts at $x$ and preserves every infinitesimal tangent
direction.  For sufficiently small $\eta$, we take
$\mathcal R_x\bigl(\eta V_{p,q}(x)\bigr)$ as the practical update.

\begin{purplepropositionbox}
\begin{proposition}[Local consistency of the retraction step]
\label{prop:retraction-consistency}
Let $\mathcal R$ be a $C^2$ retraction.  For fixed $p$, $q$, and $x$, as
$\eta\downarrow0$,
\begin{equation}
  d_{\mathcal M}\!\left(
  \mathcal R_x\!\bigl(\eta V_{p,q}(x)\bigr),
  \operatorname{Exp}_x\!\bigl(\eta V_{p,q}(x)\bigr)
  \right)=O(\eta^2).
  \label{eq:general-retraction-consistency}
\end{equation}
\end{proposition}
\end{purplepropositionbox}
Thus the retraction step stays on $\mathcal M$ and reproduces the intrinsic
motion to first order.  Concrete choices of $\mathcal R_x$ for each ground
geometry considered in this work are given in
Appendix~\ref{app:concrete-geometries}.  The proof and additional details are
given in Appendix~\ref{app:proof:finite-update}.

\paragraph{Finite-sample fields.}
In practice, we do not have analytic access to $p$ and $q$ or to the
population expectations defining $V_{p,q}$, we only have target samples and
samples from the current generator.  Following prior sample-based WGF and
drifting methods~\citep{arbel2019mmd,he2026sinkhorndrifting,han2026one}, we
draw independent target and model banks and form the empirical distributions
\begin{equation}
  \widehat p:=\frac1{N_p}\sum_{j=1}^{N_p}\delta_{y_j^+},
  \qquad
  \widehat q:=\frac1{N_q}\sum_{j=1}^{N_q}\delta_{y_j^-},
  \qquad
  y_j^+\overset{\mathrm{i.i.d.}}{\sim}p,\quad
  y_j^-\overset{\mathrm{i.i.d.}}{\sim}q,
  \label{eq:empirical-banks}
\end{equation}
where $\delta_y$ is the Dirac mass at $y$.  For MMD and Sinkhorn, whose
objectives accept atomic inputs, these empirical laws induce a plug-in field,
which we denote by $\widehat V_{p,q} := V_{\widehat p,\widehat q}$. However, a direct plug-in is not available for KL.  The atomic measures $\widehat p$
and $\widehat q$ are singular with respect to Riemannian volume, so they have
neither positive densities nor ordinary scores:
$\operatorname{grad}_{\mathsf g}\log\widehat p$ and
$\operatorname{grad}_{\mathsf g}\log\widehat q$ are
undefined, and disjoint empirical supports typically give
$\operatorname{KL}(\widehat q\,\|\,\widehat p)=+\infty$.  Following prior
Parzen-smoothed score estimation of drifting fields
~\citep{gretton2026interpretation,han2026one},
we instead choose a smoothing kernel $\kappa_\rho$ and define
\begin{equation}
  \begin{aligned}
    \widehat p_\rho(x)&:=\frac1{N_p}\sum_{j=1}^{N_p}\kappa_\rho(x,y_j^+),
    &\widehat q_\rho(x)&:=\frac1{N_q}\sum_{j=1}^{N_q}\kappa_\rho(x,y_j^-).
  \end{aligned}
  \label{eq:parzen-kl-field}
\end{equation}
This gives a smoothed-score approximation to the population KL field, not the
exact KL drift of the atomic empirical laws.  Geometry-specific choices of
kernels and costs are described in
Appendix~\ref{app:concrete-geometries}.

\paragraph{Neural amortization.}
We now use the proposed framework to train a generative model.  During
training, we discretize the Riemannian WGF with the retraction step and replace the population
drift $V_{p,q}$ by the sample-based field $\widehat V_{p,q}$.  These steps provide stopped training
targets, allowing the generator to absorb the iterative refinement while
retaining one-pass inference.

More concretly, let $p_0$ be a latent distribution on $\mathcal Z$ and
$f_\theta:\mathcal Z\to\mathcal M$ the generator, inducing
$q_\theta:=(f_\theta)_\#p_0$.  For $z\sim p_0$, define
\begin{equation}
  x=f_\theta(z),
  \qquad
  \widetilde x
  =\operatorname{sg}\!\left[
    \mathcal R_x\!\left(\eta\widehat V_{p,q_\theta}(x)\right)
  \right].
  \label{eq:teacher-target}
\end{equation}
The target $\widetilde x$ is held fixed by stop-gradient. Then we update our generator with
\begin{equation}
  \mathcal L(\theta)
  :=\mathbb E_{z\sim p_0}\!\left[
    d_{\mathcal M}^2\!\left(
      x,\widetilde x
    \right)
  \right].
  \label{eq:tracking-objective}
\end{equation}
The pseudo-code of training
procedure is presented in Algorithm~\ref{alg:geometric-drifting}.

\begin{figure}[t]
  \noindent
  \begin{minipage}[t]{0.44\linewidth}
    \vspace{0pt}
    \hrule
    \vspace{2pt}
    \refstepcounter{algorithm}
    \label{alg:geometric-drifting}
    \addcontentsline{loa}{algorithm}{\protect\numberline{\thealgorithm}R-WGF training.}
    \noindent\textbf{Algorithm \thealgorithm\ R-WGF training.}\par
    \vspace{2pt}
    \hrule
    \vspace{3pt}
    \small\ttfamily
    \setlength{\parindent}{0pt}
    \textcolor{algorithmcomment}{\# $f_\theta$: trainable model}\\
    \textcolor{algorithmcomment}{\# functional: $\mathcal F_p$}\\
    \textcolor{algorithmcomment}{\# target distribution: $p$}\\[3pt]
    $z$ = \textcolor{algorithmfunction}{sample}($p_0$); $z'$ = \textcolor{algorithmfunction}{sample}($p_0$)\\
    $x = f_\theta(z)$\\
    $y^-$ = \textcolor{algorithmfunction}{stopgrad}($f_\theta(z')$)\\
    $y^+$ = \textcolor{algorithmfunction}{sample}($p$)\\
    V = \textcolor{algorithmfunction}{compute\_V}($\mathcal F_p$, $x$, $y^+$, $y^-$)\\
    $x_{\mathrm{target}}$ = \textcolor{algorithmfunction}{stopgrad}(%
      \textcolor{algorithmfunction}{step}($x$, $\eta$, V))\\[3pt]
    $\mathcal L(\theta)$ = \textcolor{algorithmfunction}{d\_loss}(%
      $x$, $x_{\mathrm{target}}$)\par
    \vspace{3pt}
    \hrule
  \end{minipage}\hfill
  \begin{minipage}[t]{0.55\linewidth}
    \vspace{0pt}
    \centering
    \includegraphics[width=0.86\linewidth]{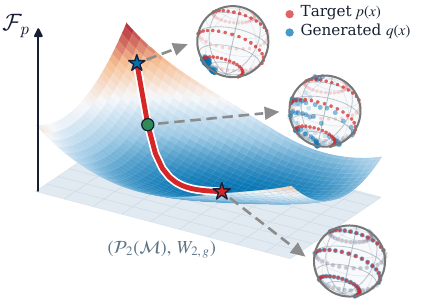}
  \end{minipage}
  \caption{\textbf{Left:} R-WGF training algorithm. \textbf{Right:} Visualization of training dynamics of
  $q$ on $\mathcal M=\mathbb S^2$ toward a spherical target $p$.}
  \label{fig:objective-geometry-interface}
\end{figure}

\subsection{Training in Feature Space}
\label{sec:method:practical}

Thus far, the costs and kernels in $\mathcal F_p$ act directly on $\mathcal M$.
As in Drifting Models~\citep{deng2026drifting}, the formulation extends to any
feature space.  Let $\phi:\mathcal M\to\mathbb R^k$ be a frozen feature
extractor applied to real and generated samples.  We compute the objective
between $\phi_\#q$ and $\phi_\#p$, so its first variation has the form
$\psi_{p,q}=\widetilde\psi_{p,q}\circ\phi$, where $\widetilde\psi_{p,q}$ is
estimated from $\phi(y^+)$ and $\phi(y^-)$.  The drift is still a tangent
vector on $\mathcal M$:
\begin{equation}
  V^{\phi}_{p,q}(x)
  =-\operatorname{grad}_{\mathsf g}\bigl(\widetilde\psi_{p,q}\circ\phi\bigr)(x).
  \label{eq:feature-drift}
\end{equation}
It is computed with one vector--Jacobian product through $\phi$, followed by
the metric conversion of Section~\ref{sec:method:drifts}.  The feature map only
changes how samples are compared.  Samples still move on $\mathcal M$, and the
loss in Eq.~\eqref{eq:tracking-objective} is unchanged.  Feature extraction is
used only at training time.  The concrete Sinkhorn, MMD, and KL drifts are
derived in Appendix~\ref{app:proofs}, and the concrete costs and kernels are
given in Appendix~\ref{app:experimental-details}.

This extends to several features $\phi_r$, $r=1,\dots,R$, for example several
layers of one encoder.  Each feature gives its own drift $V_r$.  Features have
different scales, so we normalize each drift by its norm from Eq.~\eqref{eq:law-tangent-norm}, similar
to the drift normalization of Drifting Models.  The overall loss is
\begin{equation}
  \mathcal L(\theta)
  =\sum_{r=1}^{R}w_r\,\mathbb E_{z\sim p_0}\!\left[
    d_{\mathcal M}^2\!\left(x,\,
      \operatorname{sg}\!\left[\mathcal R_x\!\left(\eta V_r(x)/\|V_r\|_{q,\mathsf g}\right)\right]
    \right)\right],
  \qquad x=f_\theta(z),\quad \sum_{r=1}^{R}w_r=1.
  \label{eq:multi-feature-loss}
\end{equation}
To first order in $\eta$, this is one step along the weighted sum
$\sum_rw_rV_r/\|V_r\|_{q,\mathsf g}$.

\paragraph{Relation to Drifting Models.}
Drifting Models regress $\phi(x)$ onto $\phi(x)+V(\phi(x))$, so both the drift
and the target live in feature space.  In R-WGF, the feature map only defines
which samples are close, while the metric $\mathsf g$ defines how $x$ moves.
The drift is pulled back to $\mathcal M$, and the target stays on
$\mathcal M$.

\section{Related Work}
\label{sec:related-work}

\paragraph{Wasserstein gradient flows and drifting models.}
Drifting Models~\citep{deng2026drifting} regress a one-step generator onto
samples displaced by a heuristic attraction--repulsion field.  Later analyses
showed that this field is a difference of smoothed scores only for Gaussian
kernels~\citep{turan2026generative,lai2026unified,franz2026drifting}, and that
the practical algorithm is not the gradient flow of any
functional~\citep{gretton2026interpretation}.  \citet{gretton2026interpretation}
make the objective-first alternative explicit: the velocity is the negative
gradient of the first variation of a chosen functional.  This route has been
followed with KDE-approximated divergences~\citep{cao2026gradient} and with the
Sinkhorn divergence~\citep{feydy2019sinkhorn} in Sinkhorn-Drifting
\citep{he2026sinkhorndrifting} and W-Flow~\citep{han2026one}. W-Flow scales it with a two-batch estimate of the self-term, which we also adopt.
Drifting has further been used to refine discrete diffusion language
models~\citep{oba2026drifting}.  All these methods are Euclidean: velocities
live in an ambient or feature space and are applied additively.  R-WGF keeps
the objective-first recipe but makes the velocity, the update step, and the
loss intrinsic to the manifold, and it recovers the fields above in the
Euclidean case (Table~\ref{tab:objective-drifts}).

\paragraph{Generation on Riemannian and categorical spaces.}
Generative models on manifolds regress intrinsic tangent fields along a
prescribed probability path~\citep{chen2024riemannian}.  For categorical data,
the Fisher--Rao geometry of the simplex underlies Fisher and statistical flow
matching~\citep{davis2024fisher,cheng2024categorical}, alongside Dirichlet and
discrete flow matching~\citep{stark2024dirichlet,gat2024discrete}.  One- and
few-step variants learn flow maps of such paths, both on
manifolds~\citep{cheng2025riemannian,davis2025generalised,woo2026riemannian,
zhong2026riemannian} and for categorical
data~\citep{roos2026categorical,potaptchik2026discrete,lee2026flow}, or build on
learned or closed-form data--noise couplings~\citep{peng2026coupling,
park2025pairflow}.  Unlike these methods, R-WGF uses the geometry as the ground
metric of a WGF~\citep{ambrosio2008gradient,erbar2010heat}, and its velocity
comes from a distributional objective without a path or a time variable.

\paragraph{Geometry-aware drifting and the position of R-WGF.}
Three works are closest to ours.  KGDM~\citep{esteban2026kernel} replaces the
Euclidean displacement in the drifting field by the kernel gradient
$\operatorname{grad}_{\mathsf g,x}\kappa_\rho(x,y)$.  With the Parzen estimates
$\widehat p_\rho$ and $\widehat q_\rho$ of Eq.~\eqref{eq:parzen-kl-field}, its
drift is $\operatorname{grad}_{\mathsf g}\log\widehat p_\rho
-\operatorname{grad}_{\mathsf g}\log\widehat q_\rho$, which is exactly the KL
branch of R-WGF.
KGDM proves that this
drift decreases a smoothed KL only in Euclidean space.
Proposition~\ref{prop:population-dissipation} shows that the exact population
flow dissipates on $\mathcal M$.
\citet{cao2026gradient}
outline a spherical variant with a von Mises--Fisher kernel.  Since
$\|x-y\|_2^2=2-2x^\top y$ on $\mathbb S^{D-1}$, it coincides with our
chordal-Gaussian KL field, but they do not specify how outputs stay on the
manifold and show only a two-dimensional example.
CatDrift~\citep{eijkelboom2026categorical} uses a negative-entropy mirror update
on the simplex, which has the same softmax form as our exponentiated retraction.  By
Propositions~\ref{prop:simplex-retraction} and~\ref{prop:retraction-consistency},
this mirror step is a valid retraction and matches the exponential-map step up
to $O(\eta^2)$.

R-WGF differs from these works in three ways.  First, it starts from the
objective rather than the kernel.  Any functional with a regular first
variation and any ground metric give the drift
$V_{p,q}=-\operatorname{grad}_{\mathsf g}\psi_{p,q}$.  Second, it yields Sinkhorn and MMD drifts on the sphere and on
Fisher--Rao simplices, which, to our knowledge, have not been derived or
evaluated before.  The Riemannian Sinkhorn drift generalizes the Euclidean
target-minus-self field~\citep{he2026sinkhorndrifting,han2026one}.  Third, it
compares MMD, Sinkhorn, and KL under one protocol on spherical, image, and
language data. Further related work is discussed in
Appendix~\ref{app:related-work}.

\section{Experiments}
\label{sec:experiments}

We evaluate our method in three settings.  First, following KGDM~\citep{esteban2026kernel}, we model geospatial data intrinsically on
$\mathbb S^2$.
Second, following PairFlow~\citep{park2025pairflow}, we model binarized MNIST and discretized CIFAR-10 on a product simplices.  Third, we
extend the same product-simplex formulation to LM1B language modeling,
building on FMLM
~\citep{lee2026flow}.  Full protocols
and extended results are in Appendix~\ref{app:experimental-details}. The KL, MMD, and Sinkhorn
variants of R-WGF are denoted R-WGF-KL, R-WGF-MMD, and R-WGF-Sinkhorn, respectively.

\subsection{Geospatial data on the sphere}
\label{sec:experiments:sphere}

We use the Earth benchmark on $\mathbb S^2$~\citep{mathieu2020riemannian} and
compare \mbox{R-WGF-MMD}, \mbox{R-WGF-KL}, and \mbox{R-WGF-Sinkhorn} with
KGDM~\citep{esteban2026kernel}. For a fair comparison, we reproduce KGDM in our training pipeline.
Theoretically, KGDM and \mbox{R-WGF-KL} are equivalent.  In practice, we implemented one variant of KGDM model
with a Laplace kernel, whereas our \mbox{R-WGF-KL}
uses a Gaussian kernel.  Section~\ref{sec:related-work}
discusses this relation.

\begin{table}[H]
  \setlength{\abovecaptionskip}{0pt}
  \setlength{\belowcaptionskip}{1pt}
  \caption{Held-out MMD on the Earth benchmark ($\downarrow$), with the three best results highlighted}
  \label{tab:earth-mmd-wgf-seed0}
  \centering
  \small
  \setlength{\tabcolsep}{2pt}
  \renewcommand{\arraystretch}{0.80}
  \setlength{\aboverulesep}{0.25ex}
  \setlength{\belowrulesep}{0.35ex}
  \resizebox{0.84\linewidth}{!}{%
  \begin{tabular}{@{}lcccc@{}}
    \toprule
    Dataset
      & \mbox{R-WGF-MMD (\textbf{Ours})}
      & \mbox{R-WGF-KL (\textbf{Ours})}
      & \mbox{R-WGF-Sinkhorn (\textbf{Ours})}
      & KGDM \\
    \midrule
    Earthquake & \cellcolor{rankblue!32}$\mathbf{0.032\pm0.0068}$
      & $0.050\pm0.0015$ & \cellcolor{rankblue!18}$0.041\pm0.0048$
      & \cellcolor{rankblue!8}$0.044\pm0.0033$ \\
    Volcano & \cellcolor{rankblue!8}$0.13\pm0.011$
      & \cellcolor{rankblue!32}$\mathbf{0.099\pm0.024}$
      & \cellcolor{rankblue!18}$0.12\pm0.047$
      & $0.14\pm0.035$ \\
    Fire & $0.036\pm0.012$
      & \cellcolor{rankblue!32}$\mathbf{0.022\pm0.0067}$
      & \cellcolor{rankblue!18}$0.031\pm0.0034$
      & \cellcolor{rankblue!8}$0.034\pm0.014$ \\
    Flood & \cellcolor{rankblue!8}$0.060\pm0.014$
      & $0.074\pm0.013$
      & \cellcolor{rankblue!32}$\mathbf{0.046\pm0.011}$
      & \cellcolor{rankblue!18}$0.054\pm0.0085$ \\
    \bottomrule
  \end{tabular}
  }
\end{table}

\begin{samepage}

Table~\ref{tab:earth-mmd-wgf-seed0} shows that no single functional dominates:
\mbox{R-WGF-MMD} is best on Earthquake, \mbox{R-WGF-KL} on Volcano and Fire,
and \mbox{R-WGF-Sinkhorn} on Flood.
\end{samepage}

\subsection{Image generation}
\label{sec:experiments:images}

Following PairFlow~\citep{park2025pairflow}, we evaluate one-step generation
on Binary-MNIST~\citep{deng2012mnist} and discretized
CIFAR-10~\citep{krizhevsky2009learning}.  Each Binary-MNIST pixel is a binary
variable, and each CIFAR-10 subpixel is a 256-way categorical variable. We follow
the FID protocol of PairFlow with argmax decoding, and compare with PairFlow,
Categorical Flow Maps (CFM)~\citep{roos2026categorical}, and
CatDrift~\citep{eijkelboom2026categorical} in
Table~\ref{tab:mnist-one-step-main}. On CIFAR-10, short pilot runs with the same recipe gave R-WGF-Sinkhorn the lowest FID,
while R-WGF-KL and R-WGF-MMD became unstable
(Appendix Fig.~\ref{fig:cifar-pilot-runs}). Because of the training cost at this scale, we therefore train only R-WGF-Sinkhorn fully.

\begin{table}[!h]
  \footnotesize
  \setlength{\belowcaptionskip}{3pt}
  \caption{Discretized image generation FID ($\downarrow$) and IS
  ($\uparrow$), with the three best results per column highlighted. Full sample grids are in
  Figs.~\ref{fig:mnist-full-sample-grid} and~\ref{fig:cifar-full-sample-grid}.}
  \label{tab:mnist-one-step-main}
  \centering
  \scriptsize
  \setlength{\tabcolsep}{1.5pt}
  \renewcommand{\arraystretch}{0.86}
  \resizebox{\linewidth}{!}{%
  \begin{tabular}[t]{@{}c@{}}
    {\bfseries Binary-MNIST}\\[2pt]
    \begin{tabular}{@{}lcr@{}}
      \toprule
      Method & NFE & FID $\downarrow$ \\
      \midrule
      PairFlow~\citep{park2025pairflow} & 4 & 8.50 \\
      PairFlow + ReDi~\citep{park2025pairflow} & 4 & 8.54 \\
      CFM~\citep{roos2026categorical} & 4
        & \cellcolor{rankblue!8}7.8 \\
      \midrule
      Euclidean Drifting~\citep{eijkelboom2026categorical}
        & 1 & $29.0\pm1.2$ \\
      CatDrift~\citep{eijkelboom2026categorical} & 1 & $8.0\pm0.1$ \\
      R-WGF-MMD (\textbf{Ours})         & 1 & $11.2\pm0.4$ \\
      R-WGF-KL (\textbf{Ours})          & 1
        & \cellcolor{rankblue!32}$\mathbf{6.5\pm0.3}$ \\
      R-WGF-Sinkhorn (\textbf{Ours})    & 1
        & \cellcolor{rankblue!18}$7.0\pm0.2$ \\
      \bottomrule
    \end{tabular}
  \end{tabular}%
  \hspace{2em}%
  \begin{tabular}[t]{@{}c@{}}
    {\bfseries Discretized CIFAR-10}\\[2pt]
    \begin{tabular}{@{}lcrr@{}}
      \toprule
      Method & NFE & FID $\downarrow$ & IS $\uparrow$ \\
      \midrule
      PairFlow~\citep{park2025pairflow} & 4 & 192.62 & 2.70 \\
      PairFlow + ReDi~\citep{park2025pairflow} & 4 & 184.45
        & \cellcolor{rankblue!8}2.91 \\
      PairFlow + DCD~\citep{park2025pairflow} & 4
        & \cellcolor{rankblue!8}139.46 & \cellcolor{rankblue!18}3.23 \\
      \midrule
      Euclidean Drifting~\citep{eijkelboom2026categorical}
        & 1 & $210\pm10.4$ & -- \\
      CatDrift~\citep{eijkelboom2026categorical} & 1
        & \cellcolor{rankblue!18}$58\pm2.9$ & -- \\
      PairFlow + ReDi~\citep{park2025pairflow} & 1 & 275.09 & 1.80 \\
      PairFlow + DCD~\citep{park2025pairflow} & 1 & 223.86 & 2.21 \\
      R-WGF-Sinkhorn (\textbf{Ours})    & 1
        & \cellcolor{rankblue!32}$\mathbf{46.9\pm0.5}$
        & \cellcolor{rankblue!32}$\mathbf{6.7\pm0.1}$ \\
      \bottomrule
    \end{tabular}
  \end{tabular}}
\end{table}

On Binary-MNIST, R-WGF-KL and R-WGF-Sinkhorn achieve the best FID with one
NFE and also outperform the four-step baselines.  R-WGF-MMD performs
comparably to the baselines.  On CIFAR-10, R-WGF-Sinkhorn reaches an FID of
46.9, below CatDrift (58) and far below all PairFlow variants, including their
four-step versions.  It also attains the highest IS.  On both datasets, the geometry-aware one-step
generators clearly outperform Euclidean Drifting, which supports respecting
categorical geometry.  Details and uncurated samples are in
Appendices~\ref{app:experimental-details:mnist}
and~\ref{app:experimental-details:cifar}.

\subsection{Language modeling}
\label{sec:experiments:language}
Following the FMLM setup~\citep{lee2026flow}, we evaluate our method on
LM1B~\citep{chelba2013one}.  Due to computational resource constraints, we initialize from the
released pretrained FLM checkpoint and fine-tune it with the R-WGF objective
rather than training from scratch. We split model development into two stages.
Stage~1 uses a limited number of training iterations to compare tuned
R-WGF-Sinkhorn and R-WGF-KL schedules. Sinkhorn achieves lower GenPPL
(Appendix Fig.~\ref{fig:lm1b-kl-sinkhorn-trajectories}), so Stage~2 continues
only R-WGF-Sinkhorn. We also tested R-WGF-MMD on LM1B in preliminary runs, but
it was not competitive, so we omit it from this screening. Table~\ref{tab:language-published-full} reports the final one-step results.
R-WGF-Sinkhorn reaches 122.54 GenPPL with 4.16 entropy.  It is thus
competitive with FMLM (119.34 and 4.16) and substantially better than CFM.
DFM (PSD) attains a lower GenPPL of 94.08, but its lower entropy of 4.06
indicates a partial collapse to less diverse text, which also lowers GenPPL.
Full details are in
Appendix~\ref{app:experimental-details:language}.

\begin{table}[!ht]
  \setlength{\belowcaptionskip}{3pt}
  \caption{One-step LM1B results, with the three best results highlighted. Full generation examples are in Appendix~\ref{app:language-qualitative}.}
  \label{tab:language-published-full}
  \centering
  \begin{minipage}[t]{0.56\linewidth}
    \vspace{0pt}
    \centering
    \footnotesize
    \setlength{\tabcolsep}{2pt}
    \renewcommand{\arraystretch}{0.86}
    \resizebox{\linewidth}{!}{%
    \begin{tabular}{@{}lrr@{}}
      \toprule
      Method & GenPPL $\downarrow$ & Entropy  $\uparrow$ \\
      \midrule
      CFM~\citep{roos2026categorical}
        & 274.87 & 3.10 \\
      DFM (PSD)~\citep{potaptchik2026discrete}
        & \cellcolor{rankblue!32}\textbf{94.08}
        & \cellcolor{rankblue!8}4.06 \\
      FMLM~\citep{lee2026flow}
        & \cellcolor{rankblue!18}119.34
        & \cellcolor{rankblue!32}\textbf{4.16} \\
      R-WGF-Sinkhorn (\textbf{Ours})
        & \cellcolor{rankblue!8}122.54	
        & \cellcolor{rankblue!32}\textbf{4.16} \\
      \bottomrule
    \end{tabular}%
    }
  \end{minipage}\hfill%
  \begin{minipage}[t]{0.41\linewidth}
    \vspace{0pt}
    \begin{tcolorbox}[
      colframe=samplegreen,
      colback=samplegreen!3!white,
      boxrule=0.55pt,
      arc=3pt,
      boxsep=0pt,
      left=3pt,
      right=3pt,
      top=2pt,
      bottom=2pt,
      middle=1pt,
      width=\linewidth,
      before skip=0pt,
      after skip=0pt,
      before upper={\raggedright\setlength{\parindent}{0pt}}
    ]
      {\scriptsize\bfseries\color{samplegreen}
      LM1B 1 step generation }
      \par\vspace{0.5pt}
      {\color{samplegreen}\hrule height 0.35pt}
      \vspace{1pt}
      {\scriptsize\itshape
      ``\ldots{} the 40 - year old man is expected to retain his parents
      and try to promote the health of them \ldots{}''
      \par\vspace{1pt}
      ``\ldots{} only both india and china are now preparing to incorporate
      workers or workers, who invariably depend on adherence to the federal
      economy \ldots{}''}
    \end{tcolorbox}
  \end{minipage}
\end{table}
\FloatBarrier

\section{Conclusion}
We introduced R-WGF, a framework that extends objective-driven Wasserstein
gradient-flow training to Riemannian state spaces.  Under the stated
assumptions, it yields intrinsic particle updates with objective dissipation
and locally consistent finite steps, then amortizes the flow into a one-step
generator.  Experiments on $\mathbb S^2$ and products of simplices show
competitive results on geospatial data, Binary-MNIST, discretized CIFAR-10 images, and LM1B.  Overall, respecting data geometry
provides a practical route to fast generation beyond Euclidean spaces.
Limitations are discussed in Appendix~\ref{app:limitations}.


\bibliography{references}
\bibliographystyle{iclr2027_conference}

\clearpage
\appendix

\section{Concrete Ground Geometries}
\label{app:concrete-geometries}

This appendix gives the geometry-specific ingredients used by the general
construction in Section~\ref{sec:method}.

\subsection{Euclidean geometry}
For $\mathcal M=\mathbb R^D$ with the standard metric,
\[
  T_x\mathcal M=\mathbb R^D,
  \qquad \mathsf g_x(u,w)=u^\top w,
  \qquad
  \operatorname{grad}_{\mathsf g}\psi=\nabla\psi,
  \qquad V_{p,q}=-\nabla\psi_{p,q}.
\]
No tangent projection is needed.  Geodesics are straight lines, hence
\[
  d_{\mathcal M}(x,y)=\|x-y\|_2,
  \qquad \operatorname{Exp}_x(v)=\mathcal R_x(v)=x+v.
\]
Thus Eq.~\eqref{eq:teacher-target} reduces to the usual additive stopped
teacher $\widetilde x=\operatorname{sg}[x+
\eta\widehat V_{p,q_\theta}(x)]$.

\subsection{Round-sphere geometry}
\label{app:geometry:sphere}
Let $D\geq2$ and
$\mathcal M=\mathbb S^{D-1}=\{x\in\mathbb R^D:\|x\|_2=1\}$ with the round
metric.  Feasible velocities must be orthogonal to the radius, so
\[
  T_x\mathbb S^{D-1}=\{v:x^\top v=0\},
  \qquad P_x=I-xx^\top.
\]
The projector $P_x$ removes the radial component.  For any smooth ambient
extension $\widetilde\psi_{p,q}$ of $\psi_{p,q}$,
\[
  \operatorname{grad}_{\mathsf g}\psi_{p,q}(x)
  =P_x\nabla\widetilde\psi_{p,q}(x),
  \qquad
  V_{p,q}(x)=-P_x\nabla\widetilde\psi_{p,q}(x).
\]
The projected vector is independent of the chosen extension.

For $x,y\in\mathbb S^{D-1}$ and $v\in T_x\mathbb S^{D-1}$, with
$r=\|v\|_2$,
\begin{equation}
  d_{\mathbb S}(x,y)=\arccos(x^\top y),
  \qquad
  \mathcal R_x(v):=\operatorname{Exp}_x(v)
  =\cos(r)x+\frac{\sin(r)}{r}v,
  \qquad
  \frac{\sin(r)}{r}\bigg|_{r=0}:=1.
  \label{eq:sphere-distance-exp}
\end{equation}
The distance is the angle between $x$ and $y$.  The exponential map rotates
$x$ through angle $r$ in the plane spanned by $x$ and $v$, thereby preserving
unit norm.  It satisfies $\mathcal R_x(0)=x$ and
$\left.\frac{d}{dt}\mathcal R_x(tv)\right|_{t=0}=v$, so it is a valid
retraction.  Substituting $v=\eta\widehat V_{p,q_\theta}(x)$ in the displayed
formula gives the teacher in Eq.~\eqref{eq:teacher-target}, which therefore
stays exactly on the sphere.

\subsection{Product-simplex Fisher--Rao geometry}
\label{app:geometry:categorical}

For $d$ categories and $m$ coordinates, define
\[
  \begin{gathered}
    \Delta^{d-1}=\{a\in\mathbb R_{\geq0}^d:\mathbf1^\top a=1\},
    \qquad
    \Delta_\circ^{d-1}=\{a\in\mathbb R_{>0}^d:\mathbf1^\top a=1\},\\
    \mathcal X=(\Delta^{d-1})^m,
    \qquad
    \mathcal M=\mathcal X_\circ=(\Delta_\circ^{d-1})^m.
  \end{gathered}
\]
A law on $\mathcal X_\circ$ is a joint law over the complete relaxed object;
the product ground metric does not assume independence across coordinates.

For a categorical object $y=(y_1,\ldots,y_m)$, let
$e_y=(e_{y_1},\ldots,e_{y_m})$ be its one-hot embedding.  We move data into
the smooth interior using
\[
  e_\tau(e_j)=(1-\tau)e_j+\frac{\tau}{d}\mathbf1,
  \qquad
  E_\tau(e_y)=\bigl(e_\tau(e_{y_1}),\ldots,e_\tau(e_{y_m})\bigr),
  \qquad
  p_\tau=(E_\tau\circ e)_\#p_{\mathrm{data}}.
\]
For the interior theory, $p=p_\tau$ with $0<\tau<1$. Experimentally,
$\epsilon_{\mathrm{data}}=\tau$, while $\tau=0$ is a boundary extension.

At $a\in\Delta_\circ^{d-1}$, the Fisher--Rao metric is
\[
  T_a\Delta_\circ^{d-1}=\{u:\mathbf1^\top u=0\},
  \qquad
  \langle u,w\rangle_{\mathrm{FR},a}
  =\sum_{i=1}^d\frac{u_iw_i}{a_i}.
\]
Its ambient cometric representation, extended blockwise to the product, is
\begin{equation}
  G_{\mathrm{FR}}(a)=\operatorname{Diag}(a)-aa^\top,
  \qquad
  G_x=\operatorname{blkdiag}_{\ell=1}^mG_{\mathrm{FR}}(x_\ell).
  \label{eq:fr-cometric}
\end{equation}
For any ambient covector representative $h$,
\begin{equation}
  [G_xh]_{\ell i}
  =x_{\ell i}\left(h_{\ell i}-\sum_jx_{\ell j}h_{\ell j}\right).
  \label{eq:fr-cometric-components}
\end{equation}
This vector is tangent blockwise and is unchanged if a constant is added to
all components of one block of $h$.  The derivation is in
Appendix~\ref{app:proof:fr-cometric}.

The block distance and product distance are
\begin{equation}
  d_{\mathrm{FR}}(a,b)
  =2\arccos\!\left(\sum_i\sqrt{a_ib_i}\right),
  \qquad
  d_{\mathcal X}^2(x,x')
  =\sum_{\ell=1}^m d_{\mathrm{FR}}^2(x_\ell,x'_\ell).
  \label{eq:fr-distance}
\end{equation}
Because the open simplex is not geodesically complete, the smooth results are
local to compact interiors
\[
  \mathcal X_\delta
  =\{x\in\mathcal X_\circ:x_{\ell i}\geq\delta\ \forall\ell,i\},
  \qquad 0<\delta<\tau/d.
\]

Let $b_{p,q}(x)$ be an ambient representative of the descent covector
$-d\psi_{p,q}(x)$.  The categorical intrinsic velocity is simply
\[
  V_{p,q}^{\mathrm{cat}}(x)=G_xb_{p,q}(x),
  \qquad
  [V_{p,q}^{\mathrm{cat}}(x)]_{\ell i}
  =x_{\ell i}\left(b_{\ell i}-\sum_jx_{\ell j}b_{\ell j}\right).
\]
The objective-specific covectors are listed in
Table~\ref{tab:objective-drifts}.  MMD and Sinkhorn accept atomic target and
model laws, and exact KL requires positive densities, so the empirical KL
experiments use the Parzen score approximation built from the smoothed
functions in Eq.~\eqref{eq:parzen-kl-field}, not the exact KL drift of those
atomic laws.
Boundary protocols are not covered by the compact-interior statements.

\paragraph{Simplex-preserving retraction.}
For $x\in\mathcal X_\circ$ and $v\in T_x\mathcal X_\circ$, define
\begin{equation}
  [\mathcal R_x(v)]_{\ell i}
  =\frac{x_{\ell i}\exp(v_{\ell i}/x_{\ell i})}
         {\sum_jx_{\ell j}\exp(v_{\ell j}/x_{\ell j})}.
  \label{eq:simplex-retraction}
\end{equation}
If $v=\eta G_xb$, the blockwise centering constant cancels and gives
\begin{equation}
  [\mathcal R_x(\eta G_xb)]_{\ell i}
  =\frac{x_{\ell i}\exp(\eta b_{\ell i})}
         {\sum_jx_{\ell j}\exp(\eta b_{\ell j})}.
  \label{eq:simplex-retraction-covector}
\end{equation}

\begin{purplepropositionbox}
\begin{proposition}[Product-simplex exponentiated retraction]
\label{prop:simplex-retraction}
Equations~\eqref{eq:simplex-retraction} and
\eqref{eq:simplex-retraction-covector} define a $C^\infty$ retraction that
preserves positivity and blockwise normalization; its covector form is
invariant to blockwise constants in $b$.  Uniformly for bounded $b$ on a
compact interior,
\begin{align}
  \mathcal R_x(\eta G_xb)
  &=x+\eta G_xb+O(\eta^2),
  &&\eta\downarrow0,
  \label{eq:simplex-retraction-first-order}\\
  d_{\mathcal X}\!\left(
    \mathcal R_x(\eta G_xb),
    \operatorname{Exp}_x(\eta G_xb)
  \right)
  &=O(\eta^2).
  \label{eq:simplex-retraction-consistency}
\end{align}
\end{proposition}
\end{purplepropositionbox}
The proof is in Appendix~\ref{app:proof:finite-update}.  This retraction is
not the exact Fisher--Rao exponential map: it preserves the open simplex for
finite $\eta$, but a prescribed margin $\mathcal X_\delta$ still requires
step-size control.

\paragraph{Objective costs and kernels.}
The Fisher--Rao distance defines the Wasserstein ground geometry.  The
separate choices used inside the Sinkhorn and MMD objectives are
\begin{align}
  c_{\mathrm{ch}}(x,y)&=\frac12\|\sqrt{x}-\sqrt{y}\|_2^2,
  \label{eq:chordal-cost}\\
  \kappa_\sigma(x,x')
  &=\exp\!\left[-\frac{\|\sqrt{x}-\sqrt{x'}\|_2^2}{2\sigma^2}\right],
  \qquad \sigma>0.
  \label{eq:sqrt-kernel}
\end{align}
On compact interiors, the Gaussian kernel is characteristic and the Gibbs
kernel generated by $c_{\mathrm{ch}}$ is positive universal.

\paragraph{Argmax decoding.}
The argmax decoder is
\[
  [D_{\mathrm{argmax}}(x)]_\ell
  =\operatorname*{arg\,max}_{i\in[d]}x_{\ell i}.
\]
For a relaxed law $q$, it induces $(D_{\mathrm{argmax}})_\#q$.  Although $K$ samples positions
conditionally independently given $x$, mixing over the joint $q$ need not
factorize.  For $0\leq\tau<1$,
\[
  D_{\mathrm{argmax}}(E_\tau(e_y))=y.
\]
Hence argmax maps $p_\tau$ exactly back to $p_{\mathrm{data}}$.

\section{Experimental details}
\label{app:experimental-details}

\subsection{Geospatial data on the sphere}
\label{app:experimental-details:sphere}

\paragraph{Data and optimization.}
We use the earthquake, volcano, fire, and flood datasets from the Earth
benchmark~\citep{mathieu2020riemannian}.  Latitude--longitude pairs are mapped
to unit vectors
$(\cos\phi\cos\lambda,\cos\phi\sin\lambda,\sin\phi)\in\mathbb S^2$.
A seed-42 split assigns approximately $80\%/10\%/10\%$ of each dataset to
training, validation, and test, only the training split is repeated to obtain
adequate minibatches.  Batches contain 1,024 points.

The generator maps $z\sim\mathcal N(0,I_3)$ through a width-32 SiLU MLP and
normalizes its output onto $\mathbb S^2$.  We use AdamW with
$\beta=(0.9,0.95)$, weight decay $10^{-2}$, 2,500 warmup steps, gradient
clipping at one, and EMA decay $0.999$.  The EMA checkpoint with the lowest
mean validation MMD over three draws is retained, with validation every 500
steps and a budget of at most 20 epochs.  Dataset-specific settings are
summarized below. $\gamma_{\mathrm{MMD}}$ and $\gamma_{\mathrm{KL}}$ are the selected
chordal-Gaussian bandwidth parameters.

\begin{table}[!ht]
  \caption{Final Earth settings.}
  \label{tab:earth-hyperparameters}
  \centering
  \small
  \setlength{\tabcolsep}{8pt}
  \renewcommand{\arraystretch}{1.05}
  \begin{tabular}{@{}lrrrr@{}}
    \toprule
    Dataset & LR & $T_{\mathrm{KGDM}}$ & $\varepsilon_{\mathrm{OT}}$
      & $\gamma_{\mathrm{MMD}}/\gamma_{\mathrm{KL}}$ \\
    \midrule
    Earthquake & $5.17\!\times\!10^{-4}$ & .7178 & .0322 & 2/2 \\
    Volcano    & $5.24\!\times\!10^{-4}$ & .3250 & .0264 & 1/1 \\
    Fire       & $9.70\!\times\!10^{-4}$ & .6332 & .2005 & 1/1 \\
    Flood      & $5.52\!\times\!10^{-4}$ & .1291 & .0333 & 4/4 \\
    \bottomrule
  \end{tabular}
\end{table}

\paragraph{Computational resources.}
All Earth runs were executed using a single NVIDIA RTX PRO 6000.

\paragraph{Implemented fields.}
For every live query batch $Q$, we draw a target batch $P$ and an independent,
detached model batch $Q'$.  R-WGF-Sinkhorn uses balanced entropic OT with

\[
  c(x,y)=\tfrac12 d_{\mathbb S^2}(x,y)^2
\]

and subtracts the $Q$--$Q'$ self field from the $Q$--$P$ cross field.
The Sinkhorn solves use log-domain iterations, marginal tolerance $10^{-6}$,
and a 5,000-iteration cap.
R-WGF-MMD minimizes half the squared MMD with

\[
  k_\gamma(x,y)=\exp[-\gamma\|x-y\|_2^2]
\]

and estimates the target-attraction and model-repulsion terms by arithmetic
bank averages.  R-WGF-KL uses the same kernel to form Parzen estimates
$\widehat p_\gamma$ and $\widehat q_\gamma$, and follows the spherical score
difference
$\operatorname{grad}_{\mathbb S^2}\log\widehat p_\gamma-
 \operatorname{grad}_{\mathbb S^2}\log\widehat q_\gamma$.

All three fields are evaluated in FP32.  With nominal step size $\eta=1$, we
cap each tangent velocity at the dataset-specific $\delta_{\max}$, apply the
spherical exponential map, detach
the resulting teacher, and fit it with squared geodesic distance.  The matched
KGDM control~\citep{esteban2026kernel} uses the same data, generator,
optimizer, and validation protocol, but retains its geodesic Laplace kernel
and same-query leave-one-out reference construction.

\subsection{Binary-MNIST image generation}
\label{app:experimental-details:mnist}

\paragraph{Representation and common training setup.}
An image is represented as
$x=(x_1,\ldots,x_{784})\in(\Delta^1)^{784}$.  The generator is a
$6.34$M-parameter CatDrift-style patch Transformer with $4\times4$ patches,
width 256, six blocks, four attention heads, eight registers, four noise
channels, and two output categories.  Training uses batches of 4,096, AdamW
with learning rate $2\times10^{-4}$ and $\beta=(0.9,0.999)$, 2,000 warmup
updates, gradient clipping at one, bfloat16 generator evaluation, FP32 fields,
and EMA decay $0.9999$.

For all three objectives, the Fisher--Rao cometric converts the estimated
descent covector into a blockwise tangent velocity.  We use the common teacher

\begin{equation}
  \widetilde x
  =\mathcal N\!\left(
    \max\{x+0.05\,\widehat V^{\mathrm{cat}}(x),10^{-8}\}
  \right),
  \label{eq:mnist-projected-euler}
\end{equation}

where $\mathcal N$ renormalizes every binary block, and fit its stopped value
with the sum of sitewise squared-Hellinger losses.  This clipped projected
update is an empirical boundary discretization, rather than the exact
Fisher--Rao exponential step covered by the compact-interior analysis.

Define the mean sitewise Hellinger cost and the frozen perceptual cost by

\[
  \overline H^2(x,y)
  :=\frac1{784}\sum_{s=1}^{784}
  \frac12\|\sqrt{x_s}-\sqrt{y_s}\|_2^2,
  \qquad
  c_{\mathrm{hyb}}(x,y)
  :=\overline H^2(x,y)+\|\bar\phi(x)-\bar\phi(y)\|_2.
\]

Here $\bar\phi$ is the normalized, mean-pooled representation of a frozen
masked autoencoder evaluated on unmasked images.  The objective-specific
settings are listed below.  For a one-hot target $y$,
$\epsilon_{\mathrm{data}}$ denotes the relaxation
$y\leftarrow(1-\epsilon_{\mathrm{data}})y+
(\epsilon_{\mathrm{data}}/2)\mathbf 1_2$ and zero denotes a hard boundary target.

\begin{table}[!ht]
  \caption{Binary-MNIST objective implementations.}
  \label{tab:mnist-objective-settings}
  \centering
  \small
  \setlength{\tabcolsep}{5pt}
  \renewcommand{\arraystretch}{1.08}
  \begin{tabular}{@{}lcc@{}}
    \toprule
    Method & $\epsilon_{\mathrm{data}}$ & Scale \\
    \midrule
    R-WGF-Sinkhorn & 0 & $\varepsilon_{\mathrm{OT}}=.005$ \\
    R-WGF-KL & $.001$ & $\rho=.05$ \\
    R-WGF-MMD & $.001$ & $\rho=.05$ \\
    \bottomrule
  \end{tabular}
\end{table}

Sinkhorn uses 4,096 real and model particles. Its solver uses at most 1,000 iterations,
tolerance $10^{-4}$, and convergence checks every five iterations.
KL and MMD use 4,096 live, target, and
independent detached model particles.

\paragraph{Training lineage and checkpoint selection.}
The reported Sinkhorn model continues a 26,000-step Hellinger-only run:
hard data and the hybrid MAE cost are introduced together, and the EMA
checkpoint at global step 36,000 is retained.  The KL model continues a
38,000-step Hellinger-kernel score run, adds the MAE term, and retains step
42,000.  These two results are therefore continuations rather than
from-scratch ablations of the final objective.  R-WGF-MMD is trained from
scratch for 50,000 updates and selects step 40,000. 

\paragraph{FID protocol.}
We follow the PairFlow-compatible protocol~\citep{park2025pairflow}: argmax
decode 1,000 EMA samples and compare them with all 60,000 Binary-MNIST training
images.  After checkpoint selection, each fixed checkpoint is evaluated on
nine new sampling seeds.  Consequently, the standard deviations in
Table~\ref{tab:mnist-one-step-main} measure sampling variability for one
trained checkpoint, not variability across training runs.

\paragraph{Computational resources.}
All Binary-MNIST training jobs use one NVIDIA RTX PRO 6000.

\begin{figure}[!ht]
  \centering
  \begin{minipage}[t]{0.31\linewidth}
    \centering
    \includegraphics[width=\linewidth]
      {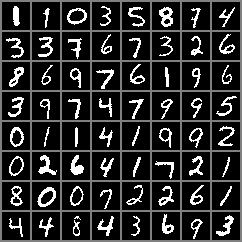}\\[-1pt]
    \small R-WGF-Sinkhorn
  \end{minipage}\hfill
  \begin{minipage}[t]{0.31\linewidth}
    \centering
    \includegraphics[width=\linewidth]
      {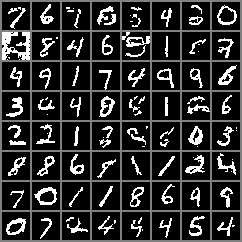}\\[-1pt]
    \small R-WGF-MMD
  \end{minipage}\hfill
  \begin{minipage}[t]{0.31\linewidth}
    \centering
    \includegraphics[width=\linewidth]
      {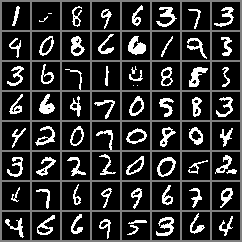}\\[-1pt]
    \small R-WGF-KL
  \end{minipage}
  \caption{\textbf{Full Binary-MNIST sample grids.}
  Fixed, uncurated $8\times8$ argmax grids from the selected EMA checkpoints.}
  \label{fig:mnist-full-sample-grid}
\end{figure}
\FloatBarrier

\subsection{Discretized CIFAR-10 image generation}
\label{app:experimental-details:cifar}

\paragraph{Representation and generator.}
An image is represented as $x\in(\Delta^{255})^{3072}$: one distribution over
256 intensities for each of the $32\times32\times3$ subpixels of
discretized CIFAR-10~\citep{krizhevsky2009learning}.  The generator is a U-Net with 128
base channels, channel multipliers $(1,2,2,2)$, two residual blocks per scale,
one attention resolution, and dropout $0.1$.  It maps Gaussian noise of image
shape to a truncated discretized logistic distribution for every
subpixel~\citep{austin2021structured}, whose scale is bounded to $[0.5,8]$
intensity bins.

\paragraph{Objective.}
We use R-WGF-Sinkhorn with the covector retraction
Eq.~\eqref{eq:simplex-retraction-covector} and the multi-feature loss
Eq.~\eqref{eq:multi-feature-loss} with equal weights, $\eta=10^{-5}$, and hard
data targets.  As for Binary-MNIST, $d_{\mathcal M}^2$ is the sum of sitewise
squared-Hellinger distances.  The feature maps are the Hellinger coordinates
on the product simplex, the expected RGB image and its channel RMS, and
multi-layer features of a frozen DINOv2 ViT-S/14~\citep{oquab2024dinov2}
applied to the expected RGB image upsampled to $224\times224$.

For each live batch, we draw a data batch and an independent, detached model
batch.  Sinkhorn is solved on blocks of $64\times64$ particles per GPU, with
the squared Euclidean feature cost divided by its batch scale and
$\varepsilon_{\mathrm{OT}}=0.05$ in these units.  The solver runs at most ten
iterations in FP64 with tolerance $10^{-5}$.  Each drift is normalized by its
Fisher--Rao RMS over the microbatch, as in
Eq.~\eqref{eq:multi-feature-loss}.

\paragraph{Optimization.}
We use Adam with learning rate $3\times10^{-4}$, $\beta=(0.9,0.999)$, no
weight decay, 5,000 warmup updates and a constant schedule afterwards,
gradient clipping at one, bfloat16 precision, and EMA decay $0.9999$.  The
global batch is 8,192 images, with 64 images per GPU and gradient
accumulation.

\paragraph{FID protocol.}
About every 500 updates, we draw 5,000 EMA samples with one NFE and argmax
decoding.  FID is computed against the statistics of the CIFAR-10 training set
with the PairFlow implementation, and we retain the checkpoint with the lowest
FID.  We then evaluate this checkpoint on five new sampling seeds, with 5,000
and with 50,000 samples, and report the mean and standard deviation.

\paragraph{Computational resources.}
Training uses eight NVIDIA RTX PRO 6000 GPUs.

\paragraph{Pilot runs.}
We trained R-WGF-Sinkhorn, R-WGF-KL,
and R-WGF-MMD from scratch with the same recipe.  KL and MMD replace the Sinkhorn plan by Parzen and kernel
weights with bandwidth $0.05$.  Figure~\ref{fig:cifar-pilot-runs} shows
FID-5k and IS during the first 3,000 updates.  Sinkhorn
improves steadily and reaches the lowest FID, 179.9 at update 2,665.  MMD is briefly better at update 1,065 but
then deteriorates, and KL deteriorates after update 2,133.  The KL run
stopped with non-finite costs.

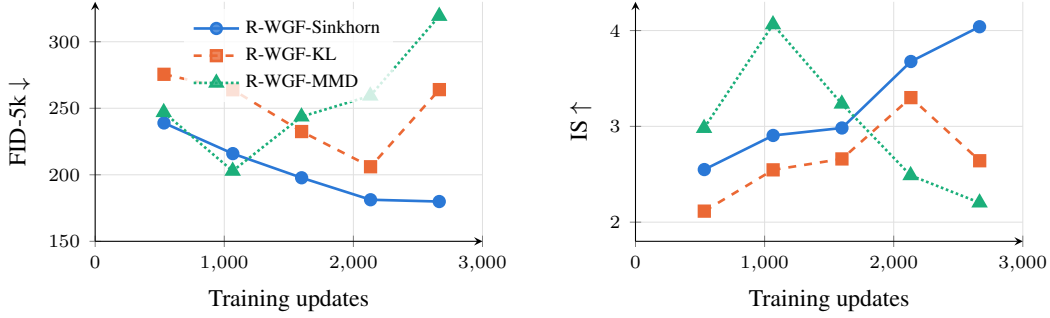
\begin{figure}[!ht]
  \centering
  \pgfplotsset{
    pilot/.style={
      width=0.48\linewidth, height=0.34\linewidth,
      xlabel={Training updates}, xmin=0, xmax=3000,
      xtick={0,1000,2000,3000},
      scaled x ticks=false,
      tick label style={font=\scriptsize}, label style={font=\small},
      axis lines=left, grid=major,
      grid style={line width=0.2pt, draw=black!12},
      every axis plot/.append style={line width=1pt, mark size=2pt},
    },
  }
  \begin{tikzpicture}
    \begin{axis}[pilot, ylabel={FID-5k $\downarrow$}, ymin=150, ymax=330,
      legend style={font=\scriptsize, draw=none, fill=white, fill opacity=0.8,
        text opacity=1, at={(0.5,0.97)}, anchor=north},
      legend cell align=left]
      \addplot[color=pilotsink, mark=*] coordinates
        {(533,238.96) (1066,215.91) (1599,197.76) (2133,181.25)
         (2665,179.90)};
      \addlegendentry{R-WGF-Sinkhorn}
      \addplot[color=pilotkl, mark=square*, dashed, mark options={solid}] coordinates
        {(533,275.53) (1066,263.93) (1599,232.50) (2133,206.04)
         (2666,264.02)};
      \addlegendentry{R-WGF-KL}
      \addplot[color=pilotmmd, mark=triangle*, mark size=2.6pt, densely dotted,
        mark options={solid}] coordinates
        {(531,246.96) (1065,202.86) (1598,243.67) (2131,259.36)
         (2665,319.27)};
      \addlegendentry{R-WGF-MMD}
    \end{axis}
  \end{tikzpicture}\hfill
  \begin{tikzpicture}
    \begin{axis}[pilot, ylabel={IS $\uparrow$}, ymin=1.8, ymax=4.3]
      \addplot[color=pilotsink, mark=*] coordinates
        {(533,2.549) (1066,2.905) (1599,2.983) (2133,3.677)
         (2665,4.040)};
      \addplot[color=pilotkl, mark=square*, dashed, mark options={solid}] coordinates
        {(533,2.115) (1066,2.546) (1599,2.660) (2133,3.300)
         (2666,2.641)};
      \addplot[color=pilotmmd, mark=triangle*, mark size=2.6pt, densely dotted,
        mark options={solid}] coordinates
        {(531,2.981) (1065,4.063) (1598,3.235) (2131,2.488)
         (2665,2.203)};
    \end{axis}
  \end{tikzpicture}
  \caption{\textbf{Discretized CIFAR-10 pilot runs.}  FID-5k (left) and IS
  (right) of R-WGF-Sinkhorn, R-WGF-KL, and R-WGF-MMD trained from scratch with
  the same recipe, evaluated with one NFE and argmax decoding.}
  \label{fig:cifar-pilot-runs}
\end{figure}

\begin{figure}[!ht]
  \centering
  \includegraphics[width=0.5\linewidth]
    {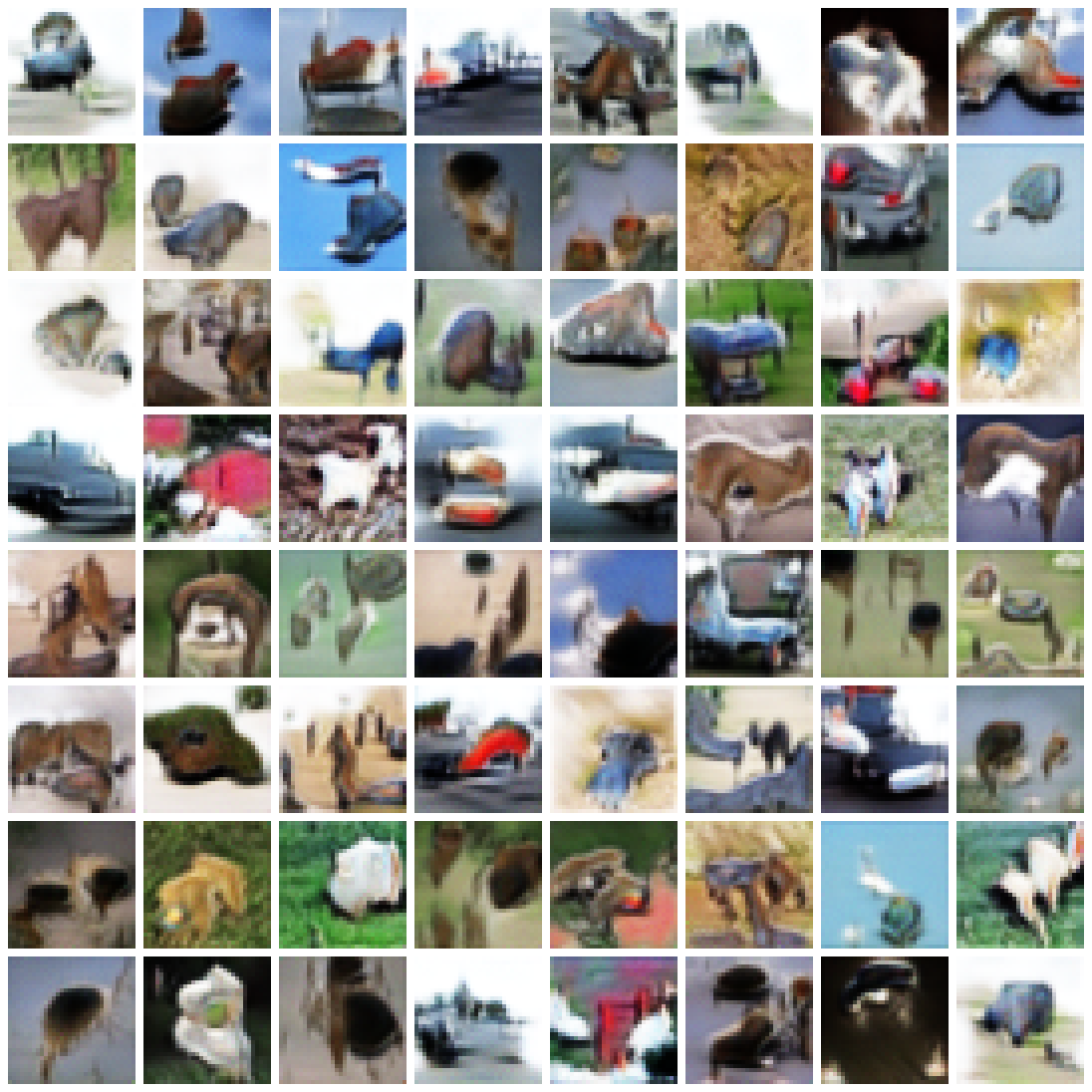}
  \caption{\textbf{Full discretized CIFAR-10 sample grid.}
  Fixed, uncurated $8\times8$ argmax grid from the selected R-WGF-Sinkhorn EMA
  checkpoint with one NFE.}
  \label{fig:cifar-full-sample-grid}
\end{figure}
\FloatBarrier

\subsection{LM1B language modeling}
\label{app:experimental-details:language}

\paragraph{Data and model.}
We use LM1B~\citep{chelba2013one} with the
\texttt{bert-base-uncased} tokenizer and length-128 packed sequences, following
the FMLM setup~\citep{lee2026flow}.  After excluding masked special and
unused entries, the active vocabulary has 29,525 tokens.  The generator is a
139M-parameter bidirectional DDiT with 12 blocks, width 768, and 12 attention
heads.  We use AdamW with learning rate $3\times10^{-5}$, 750 warmup steps,
global batch size 256, bfloat16 precision, and EMA decay $0.9999$.
The generator output
$x=(x_1,\ldots,x_{128})\in(\Delta_\circ^{29524})^{128}$ represents a joint
relaxed law over the entire sequence; the product geometry does not assume
independence between positions.

\paragraph{Computational resources.}
Stage~1 uses four and Stage~2 uses eight NVIDIA RTX PRO 6000 GPUs, with
PyTorch Lightning DDP and bfloat16 training.

\paragraph{Teacher update.}
The live batch uses an independent persistent target bank and detached model
bank.  Every implemented field is converted with the blockwise Fisher--Rao
cometric, after which we apply

\begin{align}
  \widetilde x_\ell
  &=\mathcal N_{\mathcal V}\!\left(
    \max\{x_\ell+\eta\widehat V_\ell^{\mathrm{cat}}(x),10^{-8}\}
  \right),
  \nonumber\\
  \mathcal L_{\mathrm{LM1B}}
  &=\frac1B\sum_{b=1}^B\sum_{\ell=1}^{128}
  \left\|
    \sqrt{[x_{b\ell}]_{10^{-12}}}
    -\sqrt{[\operatorname{sg}(\widetilde x_{b\ell})]_{10^{-12}}}
  \right\|_2^2.
  \label{eq:lm1b-update-loss}
\end{align}

Here $[u]_\delta=\max\{u,\delta\}$ componentwise on the active vocabulary and
$\mathcal N_{\mathcal V}$ renormalizes each token block.  As in Binary-MNIST,
this projected update and chordal loss are finite-sample approximations to the
ideal retraction and Fisher--Rao tracking objective.

\paragraph{Stage 1: objective screening.}
Starting from the same released FLM checkpoint, we compare
R-WGF-Sinkhorn and R-WGF-KL schedules.  KL uses a five-scale Parzen kernel on
sequence-normalized square-root distance, with bandwidths
$\{.12,.22,.40,.70,1.00\}$ and teacher step $\eta=.30$, its data relaxation
changes from $.075$ to $.005$ at step 15,000.  Sinkhorn uses hard
targets, $\varepsilon_{\mathrm{OT}}=.005$, and changes $\eta$ from $.03$ to
$.01$ at the same boundary.  Both use banks of capacity 1,024.  A fixed
128-sequence panel is evaluated every 1,500 updates with argmax.  Sinkhorn attains the
lower value of GenPPL with reasonable Entropy and is therefore selected for Stage 2.

\begin{figure}[!ht]
  \centering
  \includegraphics[width=0.92\linewidth]
    {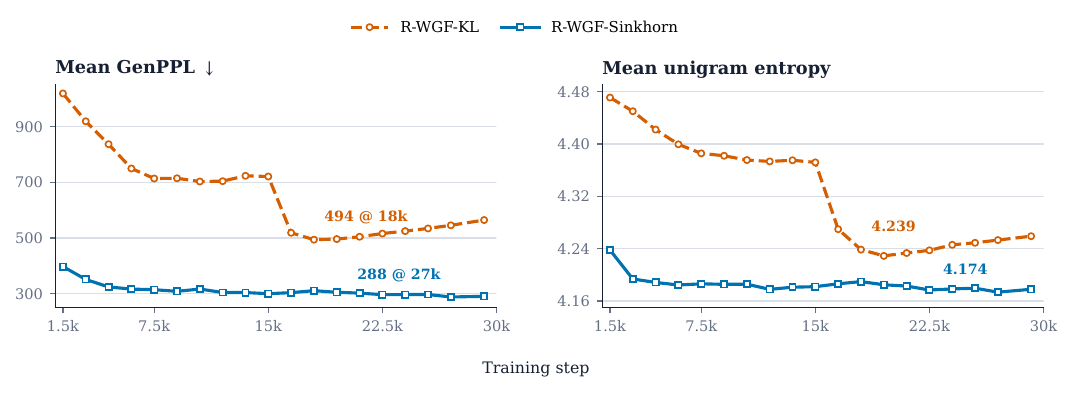}
  \caption{\textbf{LM1B objective screening.} Labels mark the lowest observed mean GenPPL.}
  \label{fig:lm1b-kl-sinkhorn-trajectories}
\end{figure}

\paragraph{Stage 2: selected Sinkhorn run.}
Stage~2 starts from the EMA weights of an earlier Stage 1 model. We use the multi-feature loss
Eq.~\eqref{eq:multi-feature-loss} with 23 equally weighted feature maps.
Twenty-one maps come from layers 4, 8, and 12 of a frozen
\texttt{bert-base-uncased} encoder, whose input at each position is the
expected token embedding under $x$.  For each layer, we use the first 127
position vectors, and their mean and standard deviation over the whole
sequence and over windows of two and four tokens.  The remaining two maps are
the Hellinger coordinates on the product simplex and their RMS.

Every map has its own Sinkhorn solve on blocks of $64\times64$ sequences.  The
squared Euclidean cost is divided by its batch scale, and
$\varepsilon_{\mathrm{OT}}=0.05$ in these units.  The solver runs at most ten
iterations in FP64.  Real sequences come from a bank of 2,048 sequences per
GPU, and the model reference is a fresh, detached batch.  We use hard target
tokens, $\eta=10^{-3}$, and a global batch of 8,192 sequences on eight GPUs,
with 64 sequences per GPU and gradient accumulation.

\paragraph{Evaluation.}
Every 1,500 updates, we evaluate 128 unconditional EMA samples with one NFE,
temperature one, argmax decoding.  GenPPL is
computed with \texttt{gpt2-large} and entropy is mean within-sequence unigram
entropy.

\subsubsection{Qualitative examples}
\label{app:language-qualitative}

\begin{quote}\small\sloppy
[CLS] during the september month allows the state party to make their own.
[CLS] when a program for health care and health care and the \$ 250, 000 money
adhere to police, who will not make their way. [CLS] the department owns its
newp - by - go system, which incorporates the federal law. [CLS] the 40 - year
old man is expected to retain his parents and try to promote the health of
them. [CLS] al per bush, the former staff chief at york york university,
jumped at a good 250 \% in december 2005 following a court case. [CLS] i think
little you should even know if i know its role and even know [CLS]
\end{quote}

\begin{quote}\small\sloppy
[CLS] per head. [CLS] it was the refusal of members of north - group and step
into iraq, though. [CLS] little africa men'department reported that this
company, however, has urged inc international group to compete at the
company's global financial security and security department, while only both
india and china are now preparing to incorporate workers or workers, who
invariably depend on adherence to the federal economy and the ability to
accelerate change. [CLS] both men are expected to be taken killed, or their
ability to take money. [CLS] india is expected to exceed 7. 4 billion percent
of 1. 4, 500 and 8 percent. [CLS] there was `` [CLS]
\end{quote}

\begin{quote}\small\sloppy
[CLS] to encourage it. [CLS] five years, he championed the life of april d.
york if you are experimenting and promotes the life of both. [CLS] they think
that it is just out across the way. [CLS] obama and the democrats vowed to
introduce it, including obama's refusal of officials to know them and the
preparing for the change. [CLS] obama, bush, the second national 500 leader,
owns the state president of new york, who was making their first win in
december 2005, following a hit and hit for him at the time of the team's 500
and 50 - 0 home game on this 8 based depending on the return [CLS]
\end{quote}

\begin{quote}\small\sloppy
[CLS] or d. bush to make a home win for the second half. [CLS] i believe that
the world points should be based at the september final weekend on obama's
adherence to the federal law, which is already my return in iraq. [CLS] obama,
based to an police spokesman and the state of south africa, has commented on
the decision. [CLS] it is also a good game, too, and a one - quarter report on
nearly out of 4, 000 - half - 2 per - downp per gop, who had only two climbs.
[CLS] the world bank made the bank deal after a refusal to compete with the
north bank [CLS]
\end{quote}

\begin{quote}\small\sloppy
[CLS] 10, 000 to the national democrats party and system, or democrats law,
system which invariably promotes the health of the use of a federal bank.
[CLS] `` the health care report, ''. d. bush, says children are killed in
south africa to the weekend of tuesday. [CLS] the 1950s clashed with the role
of a care system that ensures you get a game across the way. [CLS] both womens
descended from the little africa university. [CLS] mr. york employs members of
the national security team at a time of 50 percent, while police recently
received eight and eight half, creating a federal police report over the month
took up some money [CLS]
\end{quote}

\begin{quote}\small\sloppy
[CLS] ensuring that there may be a good way is for him. [CLS] mr york, who has
only been on experimenting, has squealed on workers from little south gop and
iran - group like iraq and iran, preferring to qualify for the good months of
40 years ; you've really know my parents. `` 40 - year - old, april, owns
president of the world's biggest 500 group in the month of 2005. [CLS] whether
dr assumes they can even all hesitate to even interfere with the role of
suchs. [CLS] the decision was recently likened by a security spokesman and
spokesman ( 1. com ) to a world bank. [CLS]
\end{quote}

\section{Additional Related Work}
\label{app:related-work}

This appendix complements Section~\ref{sec:related-work} with a detailed
comparison to inverse distillation and a broader overview of accelerated
categorical generation.

\paragraph{Inverse distillation and IDLM.}
Inverse distillation transfers iterative generation from a pretrained teacher
into the training of a few- or one-step generator.  IBMD learns a
generator-induced endpoint coupling whose bridge-matching solution reproduces
a given bridge model~\citep{gushchin2025inverse}, RealUID extends this idea to
diffusion, flow, bridge, and stochastic-interpolant models with real-data
supervision~\citep{kornilov2026universal}, and Residual Shifting Distillation
applies it to one-step super-resolution~\citep{selikhanovych2026onestep}.
Its categorical instance, IDLM~\citep{li2026idlm}, is the variant most closely
related to our work.

IDLM is the closest categorical precursor because it also optimizes an
explicit functional of the current generator law and uses simplex-valued
relaxations.  Under its ideal population assumptions, its inverse-distillation
functional is a model-to-target reverse KL between diffusion path measures
and bounds the endpoint reverse KL.  Using our endpoint notation
$(q_\theta,p)$, and writing
$x=\operatorname{softmax}(\ell)$ and teacher--fake advantage
$a_t=\log f^\star(u_t,t)-\log f(u_t,t)$, its local logit gradient contains the
same centered-simplex operator as the Fisher--Rao cometric:
\begin{align}
  \mathcal L_{\mathrm{IDLM}}
  &=\mathcal L(f^\star,q_\theta)-\min_f\mathcal L(f,q_\theta)
    =\operatorname{KL}(\mathbb P_\theta\Vert\mathbb P^\star)
    \geq \operatorname{KL}(q_\theta\Vert p),
  \nonumber\\
  \frac{\partial\mathcal L_{\mathrm{IDLM-MDLM}}}{\partial\ell_i}
  &=-\lambda_t x_i\bigl((a_t)_i-\langle x,a_t\rangle\bigr)
    =-\lambda_t[G_{\mathrm{FR}}(x)a_t]_i,
  \qquad G_{\mathrm{FR}}(x)=\operatorname{Diag}(x)-xx^\top.
\end{align}
Thus IDLM and R-WGF share an explicit law-level objective, reverse-KL
orientation, shared-latent simplex outputs, and an algebraically identical
centering operator~\citep{li2026idlm}.  The interpretation differs: IDLM
backpropagates a frozen teacher--fake advantage through generator logits,
whereas R-WGF maps an endpoint-law first variation through a chosen
Fisher--Rao cometric to an intrinsic particle velocity and then amortizes that
transport.  Our MMD and Sinkhorn fields also require no diffusion teacher.
Our executable KL branch instead uses a Parzen score approximation.  The
methods are therefore closely related, but neither is a special case of the
other.

\paragraph{Accelerated categorical generation.}
Discrete, Dirichlet, Fisher, and statistical-manifold flow matching provide
other routes to categorical generation~\citep{gat2024discrete,
stark2024dirichlet,davis2024fisher,cheng2024categorical}.  Recent work obtains
one- or few-evaluation generation using learned flow maps, direct
latent--data couplings, and diffusion--autoregressive refinement
~\citep{roos2026categorical,potaptchik2026discrete,lee2026flow,
park2025pairflow}.  Dual-Stream Decoding instead accelerates
a frozen autoregressive model at inference through parallel proposal and
refinement~\citep{abulkhanov2026dualstream}.  Our claim is therefore not
priority for one-step categorical generation, but an objective-derived
Fisher--Rao WGF on product simplices.

\section{Proofs and derivations}
\label{app:proofs}

This appendix gives short derivations of the results used in
Section~\ref{sec:method}.  Throughout, we assume the regularity needed to
differentiate under integrals and apply the transport chain rule.  Boundary
terms are taken to vanish by a no-flux or decay condition.  As in the main
text, $p$ and $q$ denote the target and model laws.

\subsection{Steepest Riemannian velocity}
\label{app:proof:steepest-velocity}

\begin{proof}[Proof of Proposition~\ref{prop:riemannian-steepest-drift}]
Let $a=\operatorname{grad}_{\mathsf g}\psi_{p,q}$ and use the
$q$-weighted inner product from Eq.~\eqref{eq:law-tangent-norm}.  By the
definition of the Riemannian gradient, the objective in
Eq.~\eqref{eq:general-riemannian-velocity} is
\begin{align*}
  J_q(u)
  &=\langle a,u\rangle_{q,\mathsf g}
    +\frac12\|u\|_{q,\mathsf g}^2\\
  &=\frac12\|u+a\|_{q,\mathsf g}^2
    -\frac12\|a\|_{q,\mathsf g}^2.
\end{align*}
The first term is minimized uniquely when $u=-a$.  Hence
\[
  V_{p,q}=-\operatorname{grad}_{\mathsf g}\psi_{p,q}
  \qquad (q\text{-a.e.}).
\]
Adding a constant to the first variation does not change this velocity because
its spatial gradient is zero.
\end{proof}

\subsection{Population dissipation}
\label{app:proof:population-dissipation}

\begin{proof}[Proof of Proposition~\ref{prop:population-dissipation}]
Set
$a_t=\operatorname{grad}_{\mathsf g}\psi_{p,q_t}
\in L^2(q_t;T\mathcal M)$.  The assumed transport chain rule and the identity
$V_{p,q_t}=-a_t$ give
\begin{align*}
  \frac{d}{dt}\mathcal F_p(q_t)
  &=\int_{\mathcal M}
    \mathsf g_x(a_t(x),V_{p,q_t}(x))q_t(x)
    \,d\mathrm{vol}_{\mathsf g}(x)\\
  &=-\int_{\mathcal M}\|a_t(x)\|_{\mathsf g,x}^2q_t(x)
    \,d\mathrm{vol}_{\mathsf g}(x)
  =-\mathcal D_p(q_t)\leq0.
\end{align*}
This is exactly Eq.~\eqref{eq:population-descent}.
\end{proof}

Provided law differentiation along $q_t$ is justified, the chain rule follows
by substituting the continuity equation and integrating by parts:
\begin{align*}
  \frac{d}{dt}\mathcal F_p(q_t)
  &=-\int_{\mathcal M}\psi_{p,q_t}
    \operatorname{div}_{\mathsf g}(q_tV_{p,q_t})
    \,d\mathrm{vol}_{\mathsf g}\\
  &=\int_{\mathcal M}
    \mathsf g(\operatorname{grad}_{\mathsf g}\psi_{p,q_t},V_{p,q_t})q_t
    \,d\mathrm{vol}_{\mathsf g},
\end{align*}
where the boundary term vanishes under the stated no-flux or decay condition.
The proposition concerns smooth density paths; it does not by itself claim
finite-step descent or convergence of the neural approximation.

\subsection{Fisher--Rao cometric and distance}
\label{app:proof:fr-cometric}

\begin{proof}[Derivation]
Fix $a\in\Delta_\circ^{d-1}$ and let $h\in\mathbb R^d$ represent a covector.
Its Riemannian-gradient vector $w$ must satisfy
\[
  \sum_i\frac{w_i u_i}{a_i}=h^\top u
  \quad\text{for every }u\text{ such that }\mathbf1^\top u=0.
\]
Therefore $w_i/a_i-h_i$ is constant across coordinates.  Writing this
constant as $-\bar h_a$ and imposing $\sum_iw_i=0$ gives
\[
  \bar h_a=\sum_i a_i h_i,
  \qquad
  w_i=a_i(h_i-\bar h_a).
\]
Equivalently,
\[
  w=(\operatorname{Diag}(a)-aa^\top)h
   =G_{\mathrm{FR}}(a)h,
\]
which proves Eqs.~\eqref{eq:fr-cometric} and
\eqref{eq:fr-cometric-components}.  The formula also makes tangency and
invariance to $h\mapsto h+c\mathbf1$ immediate.

For the distance, use the square-root map
\[
  \Phi(a)=2\sqrt a.
\]
Since $D\Phi_a[u]_i=u_i/\sqrt{a_i}$,
\[
  \langle D\Phi_a[u],D\Phi_a[v]\rangle
  =\sum_i\frac{u_iv_i}{a_i}
  =\langle u,v\rangle_{\mathrm{FR},a}.
\]
Thus $\Phi$ is an isometry onto the positive orthant of the sphere of radius
$2$.  The spherical angle between $\Phi(a)$ and $\Phi(b)$ is
$\arccos(\sum_i\sqrt{a_ib_i})$.  Because $a$ and $b$ are positive, the shorter
great-circle arc also remains in the positive orthant.  It is therefore an
admissible simplex path, so
\[
  d_{\mathrm{FR}}(a,b)
  =2\arccos\!\left(\sum_i\sqrt{a_ib_i}\right).
\]
For a product of simplices, energies add blockwise, yielding the product
distance in Eq.~\eqref{eq:fr-distance}.
\end{proof}

\subsection{Sinkhorn first variation and intrinsic drift}
\label{app:proof:sinkhorn-drift}

Let $c$ be a symmetric $C^1$ cost on a compact domain
$\Omega\subseteq\mathcal M$.  For probability laws $\alpha$ and $\beta$ on
$\Omega$, define
\begin{equation}
  \operatorname{OT}_\varepsilon(\alpha,\beta)
  :=\inf_{\pi\in\Pi(\alpha,\beta)}
  \left\{
    \int c(x,y)\,d\pi(x,y)
    +\varepsilon\operatorname{KL}(\pi\,\|\,\alpha\otimes\beta)
  \right\},
  \label{eq:entropic-ot}
\end{equation}
and the debiased Sinkhorn divergence
\begin{equation}
  \mathcal S_\varepsilon(q,p)
  =\operatorname{OT}_\varepsilon(q,p)
  -\frac12\operatorname{OT}_\varepsilon(q,q)
  -\frac12\operatorname{OT}_\varepsilon(p,p).
  \label{eq:sinkhorn-divergence}
\end{equation}
We assume that the optimal dual potentials exist and depend smoothly on the
marginals; these are the standard conditions for the envelope step below
\citep{cuturi2013sinkhorn,feydy2019sinkhorn}.

\begin{proof}[Derivation]
Let $f_{\alpha,\beta}^\varepsilon$ and
$r_{\alpha,\beta}^\varepsilon$ be optimal potentials for the first and second
marginals.  The entropic dual is
\begin{align*}
  \operatorname{OT}_\varepsilon(\alpha,\beta)
  =\sup_{f,r}\bigg\{
    &\int f\,d\alpha+\int r\,d\beta\\
    &-\varepsilon\iint
    \left[
      e^{(f(x)+r(y)-c(x,y))/\varepsilon}-1
    \right]d\alpha(x)d\beta(y)
  \bigg\}.
\end{align*}
Choose the soft-$c$-transform representatives of the optimal potentials.  At
the optimum, variation with respect to $f$ and $r$ normalizes the Gibbs factor
in each marginal:
\begin{equation}
  \int e^{(f(x)+r(y)-c(x,y))/\varepsilon}\,d\beta(y)=1,
  \qquad
  \int e^{(f(x)+r(y)-c(x,y))/\varepsilon}\,d\alpha(x)=1.
  \label{eq:sinkhorn-first-normalization}
\end{equation}
These representatives extend the normalization to every query point where the
displayed integrals are finite.  The envelope theorem now gives, for admissible
zero-mass perturbations $\sigma$ and $\tau$,
\[
  D_1\operatorname{OT}_\varepsilon(\alpha,\beta)[\sigma]
  =\int f_{\alpha,\beta}^\varepsilon\,d\sigma,
  \qquad
  D_2\operatorname{OT}_\varepsilon(\alpha,\beta)[\tau]
  =\int r_{\alpha,\beta}^\varepsilon\,d\tau.
\]
Indeed, differentiating the dual objective also produces a Gibbs-normalization
term, and Eq.~\eqref{eq:sinkhorn-first-normalization} makes that term zero.

In the self-term $\operatorname{OT}_\varepsilon(q,q)$, the same perturbation
acts on both marginals.  For symmetric $c$, choose the symmetric normalization
$f_{q,q}^\varepsilon=r_{q,q}^\varepsilon$.  Therefore
\[
  D\operatorname{OT}_\varepsilon(q,q)[\sigma]
  =\int(f_{q,q}^\varepsilon+r_{q,q}^\varepsilon)\,d\sigma
  =2\int f_{q,q}^\varepsilon\,d\sigma.
\]
This factor of two cancels the factor $1/2$ in
Eq.~\eqref{eq:sinkhorn-divergence}, giving the first variation
\begin{equation}
  \psi_{p,q}^{\mathrm{Sink}}(x)
  =f_{q,p}^\varepsilon(x)-f_{q,q}^\varepsilon(x),
  \label{eq:sinkhorn-first-variation}
\end{equation}
up to an additive constant.  Its negative Riemannian gradient is
\begin{equation}
  V_{p,q}^{\mathrm{Sink}}(x)
  =\operatorname{grad}_{\mathsf g}f_{q,q}^\varepsilon(x)
  -\operatorname{grad}_{\mathsf g}f_{q,p}^\varepsilon(x).
  \label{eq:sinkhorn-velocity}
\end{equation}

The soft $c$-transform writes the first potential as
\[
  f_{\alpha,\beta}^\varepsilon(x)
  =-\varepsilon\log\int
  e^{(r_{\alpha,\beta}^\varepsilon(y)-c(x,y))/\varepsilon}
  \,d\beta(y).
\]
Differentiating this expression shows that
\begin{equation}
  \operatorname{grad}_{\mathsf g}f_{\alpha,\beta}^\varepsilon(x)
  =\int\operatorname{grad}_{\mathsf g,x}c(x,y)\,
    \pi_{\alpha,\beta}^\varepsilon(dy\mid x),
  \label{eq:sinkhorn-potential-gradient}
\end{equation}
where the conditional entropic plan is precisely the normalized Gibbs weight.
For $c(x,y)=\tfrac12d_{\mathcal M}^2(x,y)$, provided the conditional plans put
no mass on the cut locus,
$\operatorname{grad}_{\mathsf g,x}c(x,y)=-\operatorname{Log}_x(y)$.  Hence
\begin{equation}
  V_{p,q}^{\mathrm{Sink}}(x)
  =\mathbb E_{\pi_{q,p}^\varepsilon(\cdot\mid x)}
    [\operatorname{Log}_x(Y)]
  -\mathbb E_{\pi_{q,q}^\varepsilon(\cdot\mid x)}
    [\operatorname{Log}_x(X')].
  \label{eq:sinkhorn-log-velocity}
\end{equation}
This proves the Sinkhorn column of Table~\ref{tab:objective-drifts}.
\end{proof}

A converged finite-sample Sinkhorn solve gives the exact field for the empirical
laws, while a truncated solve adds numerical approximation error.

\subsection{MMD first variation and intrinsic drift}
\label{app:proof:mmd-drift}

Let $\kappa$ be a symmetric differentiable kernel and define
$m_r^\kappa(x)=\int\kappa(x,z)\,dr(z)$.  We use
\begin{equation}
  \mathcal F_p^{\mathrm{MMD}}(q)
  :=\frac12\operatorname{MMD}_\kappa^2(q,p).
  \label{eq:mmd-objective}
\end{equation}

\begin{proof}[Derivation]
Expanding the objective gives
\[
  \mathcal F_p^{\mathrm{MMD}}(q)
  =\frac12\iint\kappa\,dq\,dq
   -\iint\kappa\,dq\,dp
   +\frac12\iint\kappa\,dp\,dp.
\]
For an admissible zero-mass perturbation $q_s=q+s\sigma$, bilinearity and
symmetry yield
\[
  \mathcal F_p^{\mathrm{MMD}}(q_s)
  =\mathcal F_p^{\mathrm{MMD}}(q)
   +s\int\bigl[m_q^\kappa(x)-m_p^\kappa(x)\bigr]\,d\sigma(x)
   +\frac{s^2}{2}\iint\kappa(x,x')\,d\sigma(x)d\sigma(x').
\]
The coefficient of $s$ is the first variation, hence
\begin{equation}
  \psi_{p,q}^{\mathrm{MMD}}(x)
  =m_q^\kappa(x)-m_p^\kappa(x).
  \label{eq:mmd-first-variation}
\end{equation}
Differentiating the two kernel means with respect to $x$ gives
\begin{equation}
  \operatorname{grad}_{\mathsf g}\psi_{p,q}^{\mathrm{MMD}}(x)
  =\mathbb E_{X'\sim q}
    [\operatorname{grad}_{\mathsf g,x}\kappa(x,X')]
  -\mathbb E_{Y\sim p}
    [\operatorname{grad}_{\mathsf g,x}\kappa(x,Y)].
  \label{eq:mmd-covector}
\end{equation}
Negating this gradient gives
\begin{equation}
  V_{p,q}^{\mathrm{MMD}}(x)
  =\mathbb E_{Y\sim p}
    [\operatorname{grad}_{\mathsf g,x}\kappa(x,Y)]
  -\mathbb E_{X'\sim q}
    [\operatorname{grad}_{\mathsf g,x}\kappa(x,X')],
  \label{eq:mmd-velocity}
\end{equation}
which is the MMD column of Table~\ref{tab:objective-drifts}.
\end{proof}

If $\kappa$ is characteristic, the objective is zero exactly when $q=p$.
This identifies the global minimum, although it does not rule out other
stationary particle configurations.

\subsection{KL first variation, Fokker--Planck equation, and dissipation}
\label{app:proof:kl-drift}

Let $\Omega$ be a compact connected domain with smooth boundary, and let $p$
and $q$ be positive smooth densities:
\begin{equation}
  p,q\in C^2(\overline\Omega),
  \qquad p(x),q(x)>0.
  \label{eq:kl-density-assumption}
\end{equation}
On a noncompact domain, the same calculation requires the corresponding
integrability and decay assumptions.
With $d\lambda=d\mathrm{vol}_{\mathsf g}$, define
\begin{equation}
  \mathcal F_p^{\mathrm{KL}}(q)
  :=\operatorname{KL}(q\,\|\,p)
  =\int_\Omega q(x)\log\frac{q(x)}{p(x)}\,d\lambda(x).
  \label{eq:kl-objective}
\end{equation}

\begin{proof}[Derivation]
Let $q_s=q+sr$ be an admissible perturbation with
$\int_\Omega r\,d\lambda=0$ and $q_s>0$ for small $s$.  Direct differentiation
gives
\[
  \left.\frac{d}{ds}\operatorname{KL}(q_s\|p)\right|_{s=0}
  =\int_\Omega r(x)
    \left(\log\frac{q(x)}{p(x)}+1\right)d\lambda(x).
\]
Therefore
\begin{equation}
  \psi_{p,q}^{\mathrm{KL}}(x)
  =\log\frac{q(x)}{p(x)}+1,
  \label{eq:kl-first-variation}
\end{equation}
and the constant disappears after taking the spatial gradient:
\begin{equation}
  V_{p,q}^{\mathrm{KL}}(x)
  =-\operatorname{grad}_{\mathsf g}\log\frac{q(x)}{p(x)}
  =\operatorname{grad}_{\mathsf g}\log p(x)
   -\operatorname{grad}_{\mathsf g}\log q(x).
  \label{eq:kl-velocity}
\end{equation}

Let $q_t$ be a smooth positive density path satisfying the same assumptions
and the no-flux condition below, and put $\phi_t=\log(q_t/p)$.  Substituting
$V_{p,q_t}=-\operatorname{grad}_{\mathsf g}\phi_t$ into the continuity
equation gives
\begin{equation}
  \partial_tq_t
  =\operatorname{div}_{\mathsf g}
    (q_t\operatorname{grad}_{\mathsf g}\phi_t)
  =\Delta_{\mathsf g}q_t
   -\operatorname{div}_{\mathsf g}
    (q_t\operatorname{grad}_{\mathsf g}\log p).
  \label{eq:kl-fokker-planck}
\end{equation}
On a domain with boundary, impose the no-flux condition
\begin{equation}
  \mathsf g
  (q_t\operatorname{grad}_{\mathsf g}\phi_t,n_{\mathsf g})=0
  \qquad\text{on }\partial\Omega.
  \label{eq:kl-no-flux}
\end{equation}
It preserves total mass and removes the boundary term in integration by parts.
Consequently,
\[
  \frac{d}{dt}\operatorname{KL}(q_t\|p)
  =-\int_\Omega q_t
    \left\|\operatorname{grad}_{\mathsf g}
      \log\frac{q_t}{p}\right\|_{\mathsf g}^2d\lambda
  \leq0.
\]
If this dissipation is zero, connectedness implies
$\log(q_t/p)$ is constant; normalization then gives $q_t=p$.
\end{proof}

This calculation applies to positive densities.  The empirical KL branch uses
the Parzen-smoothed approximation in Eq.~\eqref{eq:parzen-kl-field} because
atomic empirical laws do not have ordinary density scores.

\subsection{Retraction consistency}
\label{app:proof:finite-update}

\begin{proof}[Proof of Proposition~\ref{prop:retraction-consistency}]
Fix $x$ and write $v=V_{p,q}(x)$.  In a local chart $\varphi$ around $x$, the
coordinate curves of $\mathcal R_x(\eta v)$ and
$\operatorname{Exp}_x(\eta v)$ both equal $\varphi(x)$ at $\eta=0$ and both
have derivative $D\varphi_x[v]$.  Taylor expansion therefore gives
\[
  \varphi(\mathcal R_x(\eta v))
  -\varphi(\operatorname{Exp}_x(\eta v))=O(\eta^2).
\]
The Riemannian and coordinate distances are locally equivalent, so
\[
  d_{\mathcal M}\!\left(
    \mathcal R_x(\eta v),
    \operatorname{Exp}_x(\eta v)
  \right)=O(\eta^2),
\]
which proves Eq.~\eqref{eq:general-retraction-consistency}.
\end{proof}

\begin{proof}[Proof of Proposition~\ref{prop:simplex-retraction}]
It is enough to work with one block.  For
$a\in\Delta_\circ^{d-1}$ and $u\in T_a\Delta_\circ^{d-1}$, write
\[
  [\mathcal R_a(u)]_i
  =\frac{a_i e^{u_i/a_i}}{\sum_j a_j e^{u_j/a_j}}.
\]
The denominator is strictly positive, so the map is smooth.  Every component
is positive and the components sum to one.  At $u=0$ the map returns $a$.
Along $tu$, the numerator derivative at $t=0$ is $u_i$, while the denominator
derivative is $\sum_j u_j=0$.  Hence
\[
  \left.\frac{d}{dt}\mathcal R_a(tu)\right|_{t=0}=u,
\]
so $\mathcal R$ is a retraction.

For a covector $b$, Eq.~\eqref{eq:fr-cometric-components} gives
\[
  \frac{[G_{\mathrm{FR}}(a)b]_i}{a_i}
  =b_i-\sum_j a_jb_j.
\]
The common centering term cancels after normalization, yielding
\[
  [\mathcal R_a(\eta G_{\mathrm{FR}}(a)b)]_i
  =\frac{a_i e^{\eta b_i}}{\sum_j a_j e^{\eta b_j}},
\]
which is Eq.~\eqref{eq:simplex-retraction-covector} and is unchanged by adding
a constant to all entries of $b$.

Finally, on a compact simplex interior, bounded $b$ implies that
$G_{\mathrm{FR}}(a)b$ stays in a bounded tangent set.  Uniform Taylor
expansion therefore gives
\[
  \mathcal R_a(\eta G_{\mathrm{FR}}(a)b)
  =a+\eta G_{\mathrm{FR}}(a)b+O(\eta^2).
\]
This proves Eq.~\eqref{eq:simplex-retraction-first-order}.  Applying
Proposition~\ref{prop:retraction-consistency} blockwise and using the product
distance in Eq.~\eqref{eq:fr-distance} proves
Eq.~\eqref{eq:simplex-retraction-consistency}; compactness makes the derivative
and metric-equivalence bounds uniform.
\end{proof}

\section{Limitations}
\label{app:limitations}

\paragraph{Experimental scale.}
Our empirical scope is limited by the available compute.  The Earth benchmark
and Binary-MNIST are relatively small.  On discretized CIFAR-10, we evaluate
only $32\times32$ images.
The language study is restricted
to LM1B, length-128 sequences, and a single 139M-parameter generator
initialized from a released FLM checkpoint.  We therefore do not establish
scaling to high-resolution images, longer contexts, larger backbones, or
additional language domains.

\paragraph{Population theory versus the training algorithm.}
The dissipation result applies to an exact, sufficiently smooth population
flow under the stated chain-rule and boundary assumptions.  It does not prove
descent after every finite empirical update, convergence of stochastic neural
optimization, or exact recovery of the target law by the amortized generator.
The retraction guarantee is also local and first-order, so finite step sizes
introduce discretization error.  We further leave well-posedness, convergence
rates, and global convergence of the population flow to future work.

\paragraph{Finite-sample and boundary approximations.}
MMD and Sinkhorn use empirical plug-in fields, and truncated Sinkhorn solves
add numerical error.  The practical KL branch instead estimates a
Parzen-smoothed score field.  In addition, the
Fisher--Rao analysis is stated on compact subsets of the open simplex, whereas
categorical observations lie on its boundary.  The clipped projected updates
used for Binary-MNIST and LM1B are practical boundary extensions and are not
covered by the compact-interior guarantees.  Discretized CIFAR-10 instead uses the exponentiated
retraction, but its hard one-hot data also lie on the boundary.  We also do not prove that
improving a relaxed-space objective must improve metrics after argmax decoding.

\end{document}